\documentclass[preprint,3p,12pt,authoryear]{elsarticle}

\usepackage{hyperref}

\usepackage{amssymb}
\usepackage{amsmath}
\usepackage{amsthm}

\usepackage{mathtools}

\usepackage{bbm}

\usepackage{bm}

\usepackage{autobreak}
\allowdisplaybreaks[4]

\usepackage[capitalize,noabbrev]{cleveref}

\theoremstyle{plain}
\newtheorem{theorem}{Theorem}[section]
\newtheorem{proposition}[theorem]{Proposition}
\newtheorem{lemma}[theorem]{Lemma}

\theoremstyle{definition}

\newtheorem{assumption}[theorem]{Assumption}
\theoremstyle{remark}

\newcommand{\R}{\mathbb{R}}
\newcommand{\E}{\mathbb{E}}
\newcommand{\Prob}{\mathbb{P}}
\newcommand{\C}{\mathcal{C}}
\newcommand{\D}{\mathcal{D}}
\newcommand{\G}{\mathcal{G}}
\newcommand{\cH}{\mathcal{H}}
\newcommand{\N}{\mathcal{N}}
\newcommand{\cO}{\mathcal{O}}
\newcommand{\X}{\mathcal{X}}
\newcommand{\Y}{\mathcal{Y}}
\newcommand{\Z}{\mathcal{Z}}

\newcommand{\id}{\mathrm{d}}
\newcommand{\inv}{^{-1}}
\newcommand{\T}{^\top}
\newcommand{\pg}{P_{G}}
\newcommand{\norm}[2][]{\left\Vert#2\right\Vert_{#1}}

\newcommand{\sw}{S_{w}}
\newcommand{\bone}{\bm{1}_{n}}
\newcommand{\dw}{\Delta w}
\newcommand{\dww}[1]{\Delta W_{#1}}
\newcommand{\duu}[1]{\Delta U_{#1}}
\newcommand{\fw}[2][G]{f_{#2}\left(#1\right)}
\newcommand{\mar}[3][M]{#1\left(#2,#3\right)}
\newcommand{\ml}[3][f_w]{#2_{\D,#3}\left(#1\right)}
\newcommand{\eml}[3][f_w]{\hat{#2}_{S,#3}\left(#1\right)}
\newcommand{\kl}[2][P]{KL\left(#2\Vert #1\right)}

\journal{Neural Networks}

\begin{document}

\begin{frontmatter}


\author{Tan Sun, Junhong Lin\corref{cor1}}
\cortext[cor1]{Correspondence to: Junhong Lin <junhong@zju.edu.cn>.}
\affiliation{organization={Center for Data Science, Zhejiang University},
            city={Hangzhou},
            country={P.R. China}}

\title{PAC-Bayesian Adversarially Robust Generalization Bounds for Graph Neural Network} 


%

\begin{abstract}
Graph neural networks (GNNs) have gained popularity for various graph-related tasks. However, similar to deep neural networks, GNNs are also vulnerable  to adversarial attacks.  Empirical studies have shown that adversarially robust generalization has a pivotal role in establishing effective defense algorithms against adversarial attacks. In this paper, we contribute by providing adversarially robust generalization bounds for two kinds of popular GNNs, graph convolutional network (GCN) and message passing graph neural network, using  the  PAC-Bayesian framework.
Our result reveals that spectral norm of the diffusion matrix on the graph and spectral norm of the weights  as well as the perturbation factor govern the robust generalization bounds of both models.
Our bounds are nontrivial generalizations  of the results developed in \citep{liao_pac-bayesian_2020} from the standard setting to adversarial setting while avoiding exponential dependence of the maximum node degree.
As corollaries, we derive better PAC-Bayesian generalization bounds for GCN in the standard setting, which improve the bounds in \citep{liao_pac-bayesian_2020}  by avoiding exponential dependence on the maximum node degree. 
\end{abstract}


%

\begin{keyword}
adversarially robust generalization bounds \sep PAC-Bayesian analysis \sep graph neural networks


\end{keyword}

\end{frontmatter}



\section{Introduction} \label{sc:introduction}

Graph neural network (GNN) \citep{scarselli_graph_2009} has developed rapidly in recent years, due to the excellent performance in a variety of applications with the graph-structure data. Generally, GNNs are neural networks that combine local aggregation of node features and non-linear transformations to predict the node-level labels or graph-level labels, in semi-supervised learning \citep{kipf_semi-supervised_2017} and supervised learning \citep{pan_joint_2015,pan_task_2016,zhang_end_2018, ying_2018_hierarchical}, respectively. With the rapid development of GNN, there is an increasing interest in understanding the generalization \citep{scarselli_vapnik_2018,verma_stability_2019,du_graph_2019,garg_generalization_2020,liao_pac-bayesian_2020,esser_learning_2021,ju_generalization_2023} and other theoretical properties of GNN.

\citet{szegedy_intriguing_2014} and \citet{goodfellow_explaining_2015} reveal that deep neural network (DNN) can be easily fooled by adversarial input samples and that adversarial training can improve the generalization of DNN.
GNN also has a similar neural network structure and can be vulnerable to adversarial attacks on graph data. Previous work focuses on how to attack on GNN  \citep{dai_adversarial_2018, zugner_adversarial_2018, bojchevski_adversarial_2019, zugner_adversarial_2019, liu_unified_2019} and how to defend GNN against adversarial attacks \citep{zhang_gnnguard_2020, wu_adversarial_2019, entezari_all_2020, zhu_robust_2019, tang_transferring_2020}, but the theoretical analysis is relatively limited. 


A recent line of work discusses the generalization of GNN in the standard (non-adversarial) setting, and derive certain bounds on complexity notations including  Vapnik–Chervonenkis (VC) dimension and Rademacher complexity that are used to characterize the generalization error, i.e., the disparity between the performance on training and test datasets of a model. Specifically, \citet{scarselli_vapnik_2018} derive an upper bound of order $\cO(p^4n^2)$ for the VC dimension of GNNs with sigmoid activations, where $p$ is the number of parameters and $n$ is the number of nodes.
\citet{verma_stability_2019} derive a generalization bound for single-layer graph convolutional neural network (GCN) by analyzing the algorithmic stability of stochastic gradient descent, where the bound depends on the number of algorithm iterations and the largest absolute eigenvalue of the graph convolution filter. 
\citet{du_graph_2019} consider an infinitely wide multi-layer GNN trained by gradient descent based on neural tangent kernels \citep{jacot_neural_2018}, and derive a generalization bound for this particular model. 
These two results are constrained to specific settings and do not apply to other scenarios.
\citet{garg_generalization_2020} present the pioneering data-dependent generalization bounds for GNNs employing Rademacher complexity, with a scaling factor of $\cO\left(hd^{l}\sqrt{l\log\left(hd\right)}\right)$. Here, $d$ is the maximum degree of the graph, while $h$ and $l$ denote the width and depth of the network, respectively.

In light of the earlier analysis, obtaining generalization bounds for GNNs using complexity notations could be challenging, even in the standard setting. Consequently, some studies have shifted their focus to the PAC-Bayesian framework. Based on the PAC-Bayesian framework, \citet{liao_pac-bayesian_2020} provide generalization bounds for GCN and message passing graph neural network (MPGNN), where the bounds scale as $\cO\left(ld^{\frac{l-1}{2}}\sqrt{h\log(lh)}\right)$ and  $\cO\left(ld^{l-2}\sqrt{h\log(lh)}\right)$ respectively. 
Recently, \citet{ju_generalization_2023} derive generalization bounds of order $\cO\left(l\norm[2]{\pg}^{l-1}\sqrt{h}\right)$ for MPGNN. The key idea of their analysis is to measure the stability of GNN against noise perturbations using Hessians, and requires that the nonlinear activations and the loss functions are twice-differentiable, Lipschitz-continuous and their first-order 
and second-order derivatives are both Lipschitz-continuous, which thus does not apply to the commonly used ReLU activation. 
Here,  $\norm[2]{\pg}$ is the  spectral norm of the GNN's feature diffusion matrix $\pg$ on the graph $G$, and varies between different models. Particularly, $\norm[2]{\pg}$ is less than 1 for GCN
and  can be less than $d$ for MPGNN.



In seeking to grasp the constrained robust generalization abilities of GNNs in the adversarial setting,
in this paper, we provide adversarially robust generalization bounds for two kinds of popular GNNs, GCN \citep{kipf_semi-supervised_2017} and MPGNN \citep{dai_discriminative_2016, jin_learning_2019}, using PAC-Bayesian framework. 
Our result reveals that spectral norm of the diffusion matrix on the graph, spectral and Frobenius norm of the weights, as well as the perturbation factor govern the robust generalization bounds of both models. 
Our bounds for GCN complement the findings in the standard setting for DNN 
in \citep{neyshabur_pac-bayesian_2018}, as well as the results for GCN in \citep{liao_pac-bayesian_2020}.  All the bounds of these studies have the same spectrally-normalized term of the weights and the dependencies on both depth and width:
$$ l\sqrt{h\log (lh)}\prod_{i=1}^{l}\norm[2]{W_i}^2\sum_{i=1}^{l}\frac{\norm[F]{W_i}^2}{\norm[2]{W_i}^2}. $$
To our knowledge, our derived PAC-Bayesian adversarially robust generalization bounds may be the first kind for GNN in the adversarial robustness settings.  Even in the standard setting, our bounds improve the bounds in \citep{liao_pac-bayesian_2020} from $\cO\left(ld^{\frac{l-1}{2}}\sqrt{h\log(lh)}\right)$  to $\cO(l\sqrt{h \log(lh)})$
for GCN, and  $\cO\left(ld^{l-2}\sqrt{h\log(lh)}\right)$ to
$\cO(l \norm[2]{\pg}^{l-2} \sqrt{h \log(lh)})$ 
for MPGNN, avoiding the exponential dependence on the maximum degree $d$ of the graph. Meanwhile, our results remove the extra first-order and second-order Lipschitz-continuous assumptions in \citep{ju_generalization_2023} for MPGNN,
exhibiting comparable tightness on the dependency on the depth, width and maximum degree.
In the context of adversarial robustness, the derived generalization bounds incorporate an additional factor $\epsilon$ that represents the intensity of the perturbation and is in general unavoidable.
The derived bounds in the robust setting indicate that one could achieve roughly the same generalization performance as in the non-adversarial attack setting, provided that the $l_2$-norm attack level on the graph node, denoted as $\epsilon$, is less than the norm on the node. Additionally, according to our theory, the generalization error bounds do not have a significant effect on the graph edge attack. Therefore, one could overlook the attack on the graph edge and instead focus on the optimization error caused by the graph edge attack.

To achieve the results, we follow the PAC-Bayesian framework \citep{mcallester_simplified_2003, neyshabur_pac-bayesian_2018} with necessary modification in the adversarial setting, and relate the robust generalization error on a prior distribution of models to that on a single model by adding random perturbation (that results in a small perturbed robust margin loss) to the single model parameters.  
We then proceed with the proof by constructing suitable prior distribution and random perturbations, combining with novel concise techniques, which are related to the Frobenius norm 
(instead of the spectral norm or the maximum rows' vector $2$-norm used in the previous literature)  of the node representations and related inequalities, to well estimate the sensitivity of the robust margin loss concerning the model parameters and to meet the challenging of avoiding the exponential dependence of the maximum degree.


\textbf{Notations.} We define some notations for the following analysis. Let $[k]$ be the set of first $k$ positive integers, i.e., $[k]=\{1,\dots,k\}$, $\bone\in\R^{1\times n}$ be an all-one vector,  and $\mathbbm{1}\{\cdot\}$ be the indicator function. For a vector $x$, $\norm[p]{x}$ denotes the vector $p$-norm. For a matrix $A$, $\norm[p]{A}$ denotes the operator norm induced by the vector $p$-norm, and $\norm[F]{A}$ denotes the Frobenius norm.

The rest of the paper are organized as follows. In \cref{sc:related-work}, we briefly review some other related works. In \cref{sc:preliminaries}, we present the problem setup and preliminaries. We outline the main results in \cref{sc:main_results}. Finally, we discuss the limitations and some open problems in \cref{sc:discussion}.

\section{Related Work} \label{sc:related-work}

\textbf{Generalization Bounds.} 
Classical results rely on complexity notions of neural networks to regulate generalization bounds, employing measures such as VC-dimension as seen in \citep{anthony_neural_2009}, or Rademacher complexity as discussed in \citep{bartlett_rademacher_2001}.
Norm-margin-based generalization bounds are developed using Rademacher complexity for fully-connected neural network in \citep{neyshabur_norm_2015, bartlett_spectrally-normalized_2017, golowich_size-independent_2018}.
\citet{neyshabur_norm_2015} establish a generalization bound with an exponential dependence $2^l$ on the network depth $l$ using Frobenius norms of parameters.
\citet{bartlett_spectrally-normalized_2017} showed a margin based generalization bound that depends on spectral norm
and $1$-norm of the layers of the networks, reducing the dependence on depth. \citet{golowich_size-independent_2018} 
provide bounds on their Rademacher complexity assuming norm constraints on the parameter matrix of each layer, refining the dependency on depth to $\sqrt{l}$.
\citet{neyshabur_pac-bayesian_2018} establish a PAC-Bayesian framework to bound the generalization gap  for fully-connected neural network, where the bound depends on the product of the spectral norm of the weights in
each layer as well as the Frobenius norm of the weights.
As mentioned above, the generalization bounds on GNN are derived using different tools such as VC-dimension \citep{scarselli_vapnik_2018}, stability analysis \citep{verma_stability_2019}, Rademacher complexity \citep{garg_generalization_2020} and PAC-Bayesian analysis \citep{liao_pac-bayesian_2020, ju_generalization_2023}. 

\textbf{Adversarial Attack and Defense.} \citet{zugner_adversarial_2018} focus on the adversarial attacks against GCN and provide an algorithm to generate adversarial perturbations targeting on the node features and the graph structures. \citet{dai_adversarial_2018} propose a reinforcement learning based attack method for adversarial attacks against GNN by modifying the structures of the graph. \citet{zugner_adversarial_2019} develop a general attack on the training procedure of GCN using meta-gradients. \citet{bojchevski_adversarial_2019} propose a principled strategy for adversarial attacks on unsupervised node embeddings. There are also some researches focusing on the defense methods. \citet{wu_adversarial_2019} provide a defense method by pre-processing the adjacency matrix of the graph. The methods in \citep{zhu_robust_2019, entezari_all_2020} defend against attacks by absorbing or dropping the noise information.
\citet{tang_transferring_2020} use transfer learning to improve the robustness of GNN. \citet{zhang_gnnguard_2020} develop a general algorithm to defend against a variety of training-time attacks that perturb the graph structure. 

\textbf{Adversarial Generalization in DNN.} A line of work presents the adversarially robust generalization bounds on DNNs using Rademacher complexity 
\citep{khim_adversarial_2018, yin_rademacher_2019, awasthi_adversarial_2020, gao_theoretical_2021, xiao_adversarial_2022}. 
\citet{khim_adversarial_2018} derive robust generalization bounds for linear and neural network via Rademacher complexity of a transformed function.
\citet{yin_rademacher_2019} established generalization bounds for both linear models and one-hidden-layer neural networks based on the
surrogate loss introduced in \citep{raghunathan2018certified} considering
vector-$\infty$-norm-additive perturbation.
\citet{awasthi_adversarial_2020} give upper
and lower bounds for Rademacher complexity of linear hypotheses with
vector-$r$-norm-additive perturbation for $r \geq 1$, and also provide upper bounds on Rademacher complexity of one hidden layer neural networks, directly applying to the original network.  \citet{gao_theoretical_2021} and \citet{farnia2018generalizable} show robust generalization bounds for DNNs, considering specific attacks. \citet{xiao_adversarial_2022} and \citet{mustafa2022generalization} employ two different techniques, respectively, to calculate the covering number that can be used to estimate the Rademacher complexity, for adversarial function classes of DNNs. Further,
\citet{xiao_pac-bayesian_nodate} present a robust generalization bound on DNN using PAC-Bayesian framework.

\section{Preliminaries} \label{sc:preliminaries}

In this section, we introduce the problem setup for multi-class graph classification problem with GCN and MPGNN, and preliminary results on PAC-Bayesian analysis.

\subsection{Problem Setup}
We consider a multi-class graph classification problem, in which, given an undirected graph $G$, we would like to classify it into one of $K$ classes. Specifically, the sample $z=(G,y)\in\Z$ consists of a graph $G$ and a label $y\in\Y=\{1,\dots,K\}$. Each graph $G=(X,A)$ has a feature matrix $X\in\X\subseteq\R^{n\times h_0}$ and an adjacency matrix $A\in\G\subseteq\R^{n\times n}$, where each row of $X$ is the feature vector of a node, $n$ is the number of nodes and $h_0$ is the input feature dimension. The classification model $f$ maps the graph $G$ to a $K$-dimension vector, and we take the index of the maximum element in $f(G)$ as the output $\hat{y}=\arg\max_{i\in[K]} f(G)_i$, where $f(G)_i$ is $i$-th element of $f(G)$. We assume that the training set $S=\{z_1,\dots,z_m\}$ with $m$ samples that are independent and identically distributed from an unknown data distribution $\D$, and that the classification model is from the hypothesis class $\cH$.

We focus on two kinds of GNN, GCN and MPGNN, which are the most popular variants of GNN. For simplicity, we define the model parameterized by $w$ as $f_w\in\cH:\X\times\G\rightarrow\R^{1\times K}$, where $w$ is the vectorization of all model parameters. 

\textbf{GCN.} An $l$-layer ($l>1$ otherwise the model is trivial) GCN can be defined as,
\begin{align}
	H_k &= \sigma_k\left(\pg H_{k-1} W_k\right), \ k\in[l-1], \nonumber\\
	H_l &= \frac{1}{n}\bone H_{l-1} W_l, \label{eq:gcn}
\end{align}
where $H_k\in\R^{n\times h_k}, k\in[l-1]$ are the node representations in each layer, $H_l\in\R^{1\times K}$ is the readout, and $W_k\in\R^{h_{k-1}\times h_{k}}, k\in[l]$ are the parameters in the $k$-th layer. And we let $H_0=X$. The matrix $\pg\in\R^{n\times n}$ is related to the graph structure and $\sigma_k(\cdot), k\in[l-1]$ are some element-wise non-linear mappings. Practically, for GCN, we take $\pg$ as the Laplacian of the graph, defined as $\tilde{D}^{-1/2}\tilde{A}\tilde{D}^{-1/2}$, where $\tilde{A}=A+I$ and $\tilde{D}=\mathrm{diag}\left(\sum_{j=1}^{n}\tilde{A}_{ij}, i\in[n]\right)$ is the degree matrix of $\tilde{A}$. Throughout of this paper, for GCN, $\pg$ is taken as the Laplacian $\tilde{D}^{-1/2}\tilde{A}\tilde{D}^{-1/2}$, and the non-linear mappings are chosen as ReLU, that is, $\sigma_k(x)=\max(0,x),\forall k\in[l-1]$.

\textbf{MPGNN.} There are several variants of MPGNN \citep{dai_discriminative_2016, gilmer2017neural,jin_learning_2019} and we introduce the models stated in \citep{garg_generalization_2020} which are popular in the generalization analysis of GNN. The $l$-layer ($l>1$ otherwise the model is trivial) MPGNN can be defined as,
\begin{align}
	& H_k = \phi_k\left(XU_{k} + \rho_k\left(\pg\psi_k\left(H_{k-1}\right)W_{k}\right)\right), \ k\in[l-1], \nonumber\\
	& H_l = \frac{1}{n}\bone H_{l-1}W_l, \label{eq:mpgnn}
\end{align}
where $H_k\in\R^{n\times h_k}, k\in[l-1]$ are the node representations in each layer, $H_l\in\R^{1\times K}$ is the readout, $W_k\in\R^{h_{k-1}\times h_{k}}, k\in[l]$ and $U_k\in\R^{h_{0}\times h_{k}}, k\in[l-1]$ are the parameters in the $k$-th layer. Here we let $H_0=0$. Several common designs for $\pg$ would be the adjacency matrix, the normalized adjacency matrix or the Laplacian of the graph. $\phi_k, \rho_k, \psi_k:\R^{n\times h}\rightarrow\R^{n\times h}, k\in[l-1]$ are element-wise non-linear mappings, e.g., ReLU and Tanh. In MPGNN, if we take $U_1$ as identity matrix, $U_k,2\le k\le l-1$ as zero matrices and $\rho_k,\psi_k$ as identity mappings, then we could get a GCN model with $l-1$ layers.

As defined above, we denote $w=vec\left(\left\{W_k\right\}_{k=1}^{l}\right)$ and $w=vec\left(\left\{W_k\right\}_{k=1}^{l}, \left\{U_k\right\}_{k=1}^{l-1}\right)$ in different settings for GCN and MPGNN, respectively.
Here, for a set of matrices, $vec(\cdot)$ is the vectorization of the set, which is the stack of columns for all matrices in the set.

\textbf{Margin Loss.} Given a GNN $f$, we focus on the generalization gap to evaluate the generalization performance of the model. For the multi-class classification problem, we can use the multi-class margin loss to define the generalization gap \citep{bartlett_spectrally-normalized_2017}. First, we define the margin operator of the true label $y\in[K]$ given $G\in\X\times\G$ for a GNN model $f_w$ as
\begin{align}
	& \mar{\fw{w}}{y} = \fw{w}_y - \max_{j\ne y}\fw{w}_j. \nonumber
\end{align}
And for the convenience of the following analysis, we define the margin operator of a pair of two classes $(i, j)$ given $G\in\X\times\G$ for a GNN model $f_w$ as
\begin{align}
	& \mar{\fw{w}}{i,j} = \fw{w}_i - \fw{w}_j,\ i,j\in[K]. \nonumber
\end{align}
The margin operator reflects the gap between the output for the true label and other labels. 
With this, the model classifies the graph correctly if and only if the margin operator is larger than zero. We expect the model to classify the graph correctly with a confidence, and hence, for a margin $\gamma>0$, we define the expected margin loss of the model $f_w$ over the data distribution $\D$ as the probability that the margin operator is less than the margin, that is, 
\begin{align}
	\ml{L}{\gamma} = \Prob_{(G,y)\sim\D}\left\{\mar{\fw{w}}{y}\le \gamma\right\}. \nonumber
\end{align}
The empirical margin loss over the dataset $S$ is defined as 
\begin{align}
	\eml{L}{\gamma} = \frac{1}{m}\sum_{(G,y)\in S}\mathbbm{1}\left\{\mar{\fw{w}}{y}\le \gamma\right\}. \nonumber
\end{align}
Thus, the generalization error of a model $f_w$ can be defined as the gap between the expected loss and the empirical loss, $\ml{L}{0}-\eml{L}{\gamma}$.

Similarly, for adversarial setting, we can define the robust margin loss. Adversarial examples are usually generated from an attack algorithm. Given $G$ and $f_w$, let $\delta_w(G)$ be the output set of an attack algorithm, which contains the adversarial samples of $G$ with respect to $f_w$. As the margin is expected to be larger, we choose the adversarial sample $G^*=\mathop{\arg\inf}\limits_{G'\in\delta_w(G)}\mar{\fw[G']{w}}{y}\in\delta_w(G)$ which minimizes the margin operator. With this, we can define the robust margin operator of the true label $y$ given $G$ for the GNN model $f_w$ as
\begin{align}
	\mar[RM]{\fw{w}}{y} &= \inf_{G'\in\delta_w(G)}\mar{\fw[G']{w}}{y} \nonumber\\
	&= \inf_{G'\in\delta_w(G)}\left(\fw[G']{w}_y - \max_{j\ne y}\fw[G']{w}_j\right). \nonumber
\end{align}
And similarly, we define the robust margin operator of a pair of two classes $(i, j)$ given $G\in\X\times\G$ for a model $f_w$ as
\begin{align}
	\mar[RM]{\fw{w}}{i,j} &= \inf_{G'\in\delta_w(G)}\left(\fw[G']{w}_i - \fw[G']{w}_j\right). \nonumber
\end{align}
For margin $\gamma>0$, the expected robust margin loss of the model $f_w$ over the data distribution $\D$ is defined as
\begin{align}
	\ml{R}{\gamma} = \Prob_{(G,y)\sim\D}\left\{\mar[RM]{\fw{w}}{y}\le \gamma\right\}, \nonumber
\end{align}
and the empirical robust margin loss over the dataset $S$ is defined as 
\begin{align}
	\eml{R}{\gamma} = \frac{1}{m}\sum_{(G,y)\in S}\mathbbm{1}\left\{\mar[RM]{\fw{w}}{y}\le \gamma\right\}. \nonumber
\end{align}
With this, the robust generalization error of a model $f_w$ can be defined as the gap between the robust expected loss and the robust empirical loss, $\ml{R}{0}-\eml{R}{\gamma}$.

\subsection{Background of PAC-Bayesian Analysis}
\label{sc:back-pac}

PAC-Bayesian analysis \citep{mcallester_simplified_2003} provides probability approximation correct (PAC) bound for generalization errors. In particular, the purpose of a learning process is to provide a posterior distribution $Q$ of a parameterized predictor $f_w$ over the hypothesis class $\cH$, which minimizes the expected margin loss with respect to the data distribution $\D$,
\begin{align}
	\ml[Q]{L}{\gamma}=\E_{w\sim Q}\left[\ml[f_w]{L}{\gamma}\right]. \nonumber
\end{align}
Since the true distribution of the data is unknowable, we define the empirical margin loss with respect to the data sample $S$,
\begin{align}
	\eml[Q]{L}{\gamma}=\E_{w\sim Q}\left[\eml[f_w]{L}{\gamma}\right]. \nonumber
\end{align}
PAC-Bayesian analysis provides guarantees for the gap between the expected and empirical loss. Following \citep{mcallester_simplified_2003}, we present the subsequent theorem with respect to the robust margin loss, defined as
\begin{align}
	\ml[Q]{R}{\gamma} = \E_{w\sim Q}\left[\ml[f_w]{R}{\gamma}\right], \nonumber
\end{align}
and
\begin{align}
	\eml[Q]{R}{\gamma} = \E_{w\sim Q}\left[\eml[f_w]{R}{\gamma}\right]. \nonumber
\end{align}

\begin{theorem}[\citet{mcallester_simplified_2003}]
	\label{thm:mcallester}
	Let $P$ be a prior distribution over $\cH$  that is independent of the training set. Then for any $\delta\in(0,1)$, with probability at least $1-\delta$ over the choice of the training set $S=\{z_1,\dots,z_m\}$ independently sampled from $\D$, for all distributions $Q$ over $\cH$ and any $\gamma>0$, we have
	\begin{align}
		\ml[Q]{R}{\gamma} \le \eml[Q]{R}{\gamma} + \sqrt{\frac{\kl{Q} + \log(2m/\delta)}{2(m-1)}}. \nonumber
	\end{align}
\end{theorem}
Here $\kl{Q}$ is the KL-divergence between two distributions. The result in the theorem holds for any given prior distribution $P$ (that is independent from the training set) and all posterior distributions $Q$. Therefore, we can construct specific priors and posteriors so that we can work out the bound. Additionally, the result  is related to quantities defined on distribution over models, while many learning processes learn a single model. Hence, in order to apply \cref{thm:mcallester} to a single model $f_w$, we construct a posterior distribution $Q$ by adding random perturbations $\dw$ to the learned parameters $w$, where $w$ is considered to be fixed and $\dw$ is drawn from a known distribution. We have the following lemma, which relates the loss on a distribution $Q$ with the loss on a single model.

\begin{lemma}
	\label{lem:pert-kl-bound}
	Let $\fw{w}:\X\times\G\rightarrow\R^K$ be any model with parameters $w$, and $P$ be any distribution on the parameters that is independent of the training data. For any $w$, we construct a posterior $Q(w+\dw)$ by adding a random perturbation $\dw$ to $w$, s.t.,
	\begin{align}
		& \Prob_{\dw}\Bigg\{\max_{i,j\in[K],\atop G\in\X\times\G} \big|\mar[RM]{f_{w+\dw}(G)}{i,j} - \mar[RM]{f_w(G)}{i,j}\big|<\frac{\gamma}{2}\Bigg\}\ge \frac{1}{2}. \nonumber
	\end{align}
	Then, for any $\gamma,\delta>0$, with probability at least $1-\delta$ over the choice of the training set $S=\{z_1,\dots,z_m\}$ independently sampled from $\D$, for any $w$, we have,
	\begin{align}
		\ml{R}{0} \le \eml{R}{\gamma} + \sqrt{\frac{2\kl{Q}+\log(8m/\delta)}{2(m-1)}}. \nonumber
	\end{align}
\end{lemma}
Here, the KL is evaluated with that $w$ is treated as fixed while $\dw$ is random, i.e. the distribution of $Q$ is the distribution of $\dw$ shifted by $w$.

\citet{neyshabur_pac-bayesian_2018} provide the lemma to give a margin-based generalization bound derived from the PAC-Bayesian bound. \citet{liao_pac-bayesian_2020} improve the lemma and present the lemma in a two-side form. Both of the above lemmas are proposed in the standard settings. We present \cref{lem:pert-kl-bound} of the two-side form with the robust margin loss in the adversarial settings. We provide proofs of \cref{thm:mcallester,lem:pert-kl-bound} in \cref{pf:thm-mcallester,pf:lem-pert-kl-bound}, respectively.

\subsection{Assumptions} \label{sc:assumptions}

In this part, we first introduce some necessary  assumptions and then present our main results on generalization bounds for GNN. 

First, we make the following common assumption for the graph data, which also appears in the literature and is easily validated on real-world data.
\begin{assumption}
	\label{assumption:graph}
	For any graph $G=(X,A)\in\X\times\G$, the node features are contained in a spectral norm ball with radius $B$, that is, $\norm[2]{X}\le B$.
\end{assumption}

We then introduce the following two assumptions which are related to the models.
\begin{assumption}
	\label{assumption:hidden-width}
	For both GCN and MPGNN models, the maximum layer width is $h$, that is $h_k\le h, \forall k\in[l] \cup \{0\}$.
\end{assumption}

\begin{assumption}
	\label{assumption:mpgnn}
	For any model $f_w\in\cH$ of MPGNN defined in \cref{eq:mpgnn}, the weight matrices in different layers are uniformly bounded, that is $\norm[2]{W_k}\le M_2, \forall k\in[l]$ and $\norm[2]{U_k}\le M_1, \forall k\in[l-1]$. The mappings $\phi_k, \rho_k, \psi_k, k\in[l-1]$ are all $L$-Lipschitz and centered at zero, that is $\phi_k(0)=\rho_k(0)=\psi_k(0)=0,\forall k\in[l-1]$.
\end{assumption}
Motivated by practical designs \citep{gilmer2017neural,jin2018junction}, in the above we consider MPGNN models with $W_k$ and $U_k$ sharing the same norm upper bounds for each $k$.

%

\section{Robust Generalization Bounds}
\label{sc:main_results}

In this section, we present our main results: improved generalization bounds in standard settings and adversarially robust generalization bounds for GCN and MPGNN.

\subsection{Generalization Bounds for GCN}

We first present the results for GCN in standard settings.

\begin{theorem}[Generalization Bounds for GCN]
	\label{thm:gen-gcn}
	Under Assumptions~\ref{assumption:graph} and~\ref{assumption:hidden-width}, for any $\delta,\gamma>0$, with probability at least $1-\delta$, over the choice of a training set $S=\{z_1,\dots,z_m\}$ independently sampled from $\D$, for any $w$, we have 
	\begin{align}
		& \ml{L}{0} \le \eml{L}{\gamma} + \cO\left(\sqrt{\frac{B^2 l^2 h \log(lh)\Phi(f_w) + \log(ml/\delta)}{\gamma^2 m}}\right), \nonumber
	\end{align} 
	where $\Phi(f_w)=\prod_{i=1}^{l}\norm[2]{W_i}^2\sum_{i=1}^{l}\frac{\norm[F]{W_i}^2}{\norm[2]{W_i}^2}$.
\end{theorem}
\cref{thm:gen-gcn} shows that the generalization bound of GCN does not grow with the maximal degree of the graph.

For the robust case, we consider an $\epsilon$-attack ($\epsilon>0$), where given a sample $G=(X,A)$, the attack sample $G'=(X',A')$ satisfies that $\norm[2]{X'-X}\le\epsilon$ and $A'$ can be any adjacency matrix.
\begin{theorem}[Robust Generalization Bounds for GCN]
	\label{thm:adv-gen-gcn}
	Under Assumptions~\ref{assumption:graph} and~\ref{assumption:hidden-width}, consider the $\epsilon$-attack with $\epsilon>0$, for any $\delta,\gamma>0$, with probability at least $1-\delta$ over the choice of a training set $S=\{z_1,\dots,z_m\}$ independently sampled from $\D$, for any $w$, we have 
	\begin{align}
		& \ml{R}{0} \le \eml{R}{\gamma}  + \cO\left(\sqrt{\frac{(B+\epsilon)^2 l^2 h \log(lh)\Phi(f_w) + \log(ml/\delta)}{\gamma^2 m}}\right), \nonumber
	\end{align} 
	where $\Phi(f_w)=\prod_{i=1}^{l}\norm[2]{W_i}^2\sum_{i=1}^{l}\frac{\norm[F]{W_i}^2}{\norm[2]{W_i}^2}$.
\end{theorem}
The above generalization bound in the adversarial setting is as tight as the bound in Theorem \ref{thm:gen-gcn} for the standard setting up to the $\epsilon$ factor of attack, which is in general unavoidable.

For the GCN model, we take $\pg$ as the Laplacian of the graph. As the spectral norm of the Laplacian is less than one, the bound due to the graph architecture $\|\pg\|^l$ raised in our proof  can be well controlled and we thus get a bound that is independent on the maximum degree $d$ of the graph.

The main difficulty is to well estimate the change
in the robust margin operator with a perturbation on the
parameters, adopting the spectral norm of $\pg$ rather than the maximum degree $d$. The estimating can be divided into two steps. The first step is to bound the change in the output
of the model with the perturbation. We present the following lemma which shows that the change can be bounded by the norm of the parameters without dependence on the degree of the graph.

\begin{lemma}
	\label{lem:pert-f-gcn}
	Let  $f_w\in\cH$ be a GCN model with $l$ layers. For any $w$, any perturbation $\dw=vec\left(\left\{\dww{i}\right\}_{i=1}^{l}\right)$ such that for all $i\in[l]$, $\norm[2]{\dww{i}}\le\norm[2]{W_i}/l$, for any $G=(X,A)\in\X\times\G$, the change in the output of GCN is bounded as 
	\begin{align}
		\norm[2]{\fw{w+\dw} - \fw{w}} & \le \frac{1}{\sqrt{n}}e\norm[F]{X}\norm[2]{\pg}^{l-1}\left(\prod_{i=1}^{l}\norm[2]{W_i}\right)\sum_{i=1}^{l}\frac{\norm[2]{\dww{i}}}{\norm[2]{W_i}} \nonumber\\
		& \le eB\left(\prod_{i=1}^{l}\norm[2]{W_i}\right)\sum_{i=1}^{l}\frac{\norm[2]{\dww{i}}}{\norm[2]{W_i}}. \nonumber
	\end{align}
\end{lemma}

\begin{proof}
	We denote the node representation in $j$-th ($j\le l$) layer as
	\begin{align}
		& f_{w}^{j}(G) = H_j = \sigma_j\left(\pg H_{j-1} W_j\right),\ j\in[l-1], \nonumber\\
		& f_{w}^{l}(G) = H_l = \frac{1}{n}\bone H_{l-1}W_l. \nonumber
	\end{align}
	Adding perturbation $\dw$ to the parameter $w$, that is, for the $j$-th ($j\le l$) layer, the perturbed parameters are $W_j+\dww{j}$ and denote $H_j'=f_{w+\dw}^{j}(G), j\in[l]$.
	
	\noindent\textbf{Upper Bound on the Node Representation. }For any $j<l$,
	\begin{align}
		\norm[F]{H_j} &= \norm[F]{\sigma_j\left(\pg H_{j-1}W_j\right)} \nonumber\\
		&\le \norm[F]{\pg H_{j-1}W_j} \nonumber\\
		&\le \norm[F]{\pg H_{j-1}}\norm[2]{W_j} \nonumber\\
		&\le \norm[2]{\pg }\norm[F]{H_{j-1}}\norm[2]{W_j}, \nonumber
	\end{align}
	where the first inequality holds since $\sigma_j$ is 1-Lipschitz and $\sigma_j(0)=0$, and the second and the last ones hold by \cref{prop:f-norm}. Then, using the relationship iteratively and $H_0=X$, we have
	\begin{align}
		\norm[F]{H_j} &\le \norm[2]{\pg }^{j}\norm[F]{H_{0}}\prod_{i=1}^{j}\norm[2]{W_i} \nonumber\\
		&\le \norm[F]{X}\norm[2]{\pg }^{j}\prod_{i=1}^{j}\norm[2]{W_i}. \label{eq:pf-gcn-pert-bound-1}
	\end{align}
	\textbf{Upper Bound on the Change of Node Representation. }For any $j<l$,
	\begin{align}
		\norm[F]{H_j'-H_j} &= \norm[F]{\sigma_j\left(\pg  H_{j-1}'(W_j+\dww{j})\right) - \sigma_j\left(\pg  H_{j-1}W_j\right)} \nonumber\\
		&\le \norm[F]{\pg  H_{j-1}'(W_j+\dww{j}) - \pg  H_{j-1}W_j} \nonumber\\
		&= \norm[F]{\pg  \left(H_{j-1}' - H_{j-1}\right) (W_j+\dww{j}) + \pg  H_{j-1}\dww{j}} \nonumber\\
		&\le \norm[F]{\pg  \left(H_{j-1}' - H_{j-1}\right) (W_j+\dww{j})} + \norm[F]{\pg  H_{j-1}\dww{j}} \nonumber\\
		&\le \norm[2]{\pg }\norm[F]{H_{j-1}' - H_{j-1}}\norm[2]{W_j+\dww{j}} + \norm[2]{\pg }\norm[F]{H_{j-1}}\norm[2]{\dww{j}}, \label{eq:pf-gcn-pert-bound-2}
	\end{align}
	where the first inequality holds since $\sigma_j$ is 1-Lipschitz and $\sigma_j(0)=0$, and the last one holds by \cref{prop:f-norm}.	Simplify the notations in \cref{eq:pf-gcn-pert-bound-2} as $\norm[F]{H_j'-H_j}\le a_{j-1}\norm[F]{H_{j-1}'-H_{j-1}}$ $+ b_{j-1}$ where $a_{j-1}=\norm[2]{\pg }\norm[2]{W_j+\dww{j}}$ and $b_{j-1}=\norm[2]{\pg }\norm[F]{H_{j-1}}\norm[2]{\dww{j}}$. With $H_0'-H_0=X-X=0$, we have,
	\begin{align}
		\norm[F]{H_j'-H_j} &\le \sum_{k=0}^{j-1}b_k \left(\prod_{i=k+1}^{j-1}a_i\right) \nonumber\\
		&= \sum_{k=0}^{j-1}\norm[2]{\pg }\norm[F]{H_{k}}\norm[2]{\dww{k+1}}\left(\prod_{i=k+1}^{j-1}\norm[2]{\pg }\norm[2]{W_{i+1}+\dww{i+1}}\right) \nonumber\\
		&= \sum_{k=0}^{j-1}\norm[2]{\pg }^{j-k}\norm[F]{H_k}\norm[2]{\dww{k+1}}\left(\prod_{i=k+2}^{j}\norm[2]{W_{i}+\dww{i}}\right). \nonumber
	\end{align}
	With \cref{eq:pf-gcn-pert-bound-1}, we have
	\begin{align}
		\norm[F]{H_j'-H_j} &\le \sum_{k=0}^{j-1}\norm[2]{\pg }^{j-k}\left(\norm[2]{\pg }^{k}\norm[F]{X}\prod_{i=1}^{k}\norm[2]{W_i}\right)\norm[2]{\dww{k+1}}\left(\prod_{i=k+2}^{j}\norm[2]{W_{i}+\dww{i}}\right) \nonumber\\
		&\le \norm[F]{X}\sum_{k=0}^{j-1}\norm[2]{\pg }^{j}\left(\prod_{i=1}^{k}\norm[2]{W_i}\right)\norm[2]{\dww{k+1}}\left(\prod_{i=k+2}^{j}\left(1+\frac{1}{l}\right)\norm[2]{W_{i}}\right) \nonumber\\
		&= \norm[F]{X}\sum_{k=0}^{j-1}\norm[2]{\pg }^{j}\left(\prod_{i=1}^{k+1}\norm[2]{W_i}\right)\frac{\norm[2]{\dww{k+1}}}{\norm[2]{W_{k+1}}}\left(\prod_{i=k+2}^{j}\left(1+\frac{1}{l}\right)\norm[2]{W_{i}}\right) \nonumber\\
		&=\norm[F]{X}\norm[2]{\pg }^{j}\left(\prod_{i=1}^{j}\norm[2]{W_i}\right)\sum_{k=0}^{j-1}\frac{\norm[2]{\dww{k+1}}}{\norm[2]{W_{k+1}}}\left(1+\frac{1}{l}\right)^{j-k-1} \nonumber\\
		&= \norm[F]{X}\norm[2]{\pg }^{j}\left(\prod_{i=1}^{j}\norm[2]{W_i}\right)\sum_{k=1}^{j}\frac{\norm[2]{\dww{k}}}{\norm[2]{W_{k}}}\left(1+\frac{1}{l}\right)^{j-k}, \label{eq:pf-gcn-pert-bound-3}
	\end{align}
	where the last inequality holds since $\norm[2]{W_i+\dww{i}}\le\norm[2]{W_i}+\norm[2]{\dww{i}}\le \left(1+1/l\right)\norm[2]{W_i}$ by the assumption that $\norm[2]{\dww{i}}\le\norm[2]{W_i}/l$.
	
	\textbf{Final Bound on the Readout Layer. } 
	\begin{align}
		\norm[2]{H_{l}'-H_{l}} &= \norm[2]{\frac{1}{n}\bm{1}_nH_{l-1}'(W_l+\dww{l}) - \frac{1}{n}\bm{1}_nH_{l-1}W_l} \nonumber\\
		&= \norm[2]{\frac{1}{n}\bm{1}_n \left(H_{l-1}'-H_{l-1}\right)(W_l+\dww{l}) + \frac{1}{n}\bm{1}_n H_{l-1}\dww{l}} \nonumber\\
		&\le \norm[2]{\frac{1}{n}\bm{1}_n \left(H_{l-1}'-H_{l-1}\right)(W_l+\dww{l})} + \norm[2]{\frac{1}{n}\bm{1}_n H_{l-1}\dww{l}} \nonumber\\
		&\le \norm[2]{\frac{1}{n}\bm{1}_{n}}\norm[2]{\left(H_{l-1}'-H_{l-1}\right)\left(W_l+\dww{l}\right)} + \norm[2]{\frac{1}{n}\bm{1}_{n}}\norm[2]{H_{l-1}\dww{l}} \nonumber\\
		&\le \frac{1}{\sqrt{n}}\norm[F]{H_{l-1}'-H_{l-1}}\norm[2]{W_l+\dww{l}} + \frac{1}{\sqrt{n}}\norm[F]{H_{l-1}}\norm[2]{\dww{l}}. \nonumber
	\end{align}
	Using \cref{eq:pf-gcn-pert-bound-1,eq:pf-gcn-pert-bound-3}, we have
	\begin{align}
		\norm[2]{H_{l}'-H_{l}}
		&\le \frac{1}{\sqrt{n}}\norm[2]{W_l+\dww{l}}\norm[F]{X}\norm[2]{\pg }^{l-1}\left(\prod_{i=1}^{l-1}\norm[2]{W_i}\right)\sum_{k=1}^{l-1}\frac{\norm[2]{\dww{k}}}{\norm[2]{W_{k}}}\left(1+\frac{1}{l}\right)^{l-k-1} \nonumber\\
		&\quad + \frac{1}{\sqrt{n}}\norm[2]{\dww{l}}\norm[2]{\pg }^{l-1}\norm[F]{X}\prod_{i=1}^{l-1}\norm[2]{W_i} \nonumber\\
		&\le \frac{1}{\sqrt{n}}\norm[F]{X}\norm[2]{\pg }^{l-1}\left[
			\norm[2]{\dww{l}}\prod_{i=1}^{l-1}\norm[2]{W_i} \right.\nonumber\\
			&\quad \left.
			+ \norm[2]{W_l+\dww{l}}\left(\prod_{i=1}^{l-1}\norm[2]{W_i}\right)\sum_{k=1}^{l-1}\frac{\norm[2]{\dww{k}}}{\norm[2]{W_{k}}}\left(1+\frac{1}{l}\right)^{l-k-1}
		\right] \nonumber\\
		&= \frac{1}{\sqrt{n}}\norm[F]{X}\norm[2]{\pg }^{l-1}\left(\prod_{i=1}^{l}\norm[2]{W_i}\right)\left[\frac{\norm[2]{\dww{l}}}{\norm[2]{W_l}} \right.\nonumber\\
		&\quad \left.
		+ \frac{\norm[2]{W_l+\dww{l}}}{\norm[2]{W_l}}\sum_{k=1}^{l-1}\frac{\norm[2]{\dww{k}}}{\norm[2]{W_{k}}}\left(1+\frac{1}{l}\right)^{l-k-1} \right] \nonumber\\
		&\le \frac{1}{\sqrt{n}}\norm[F]{X}\norm[2]{\pg }^{l-1}\left(\prod_{i=1}^{l}\norm[2]{W_i}\right)\left[\frac{\norm[2]{\dww{l}}}{\norm[2]{W_l}} \right. \nonumber\\
		&\quad \left. + \left(1+\frac{1}{l}\right)\sum_{k=1}^{l-1}\frac{\norm[2]{\dww{k}}}{\norm[2]{W_{k}}}\left(1+\frac{1}{l}\right)^{l-k-1}\right] \nonumber\\
		&= \frac{1}{\sqrt{n}}\norm[F]{X}\norm[2]{\pg }^{l-1}\left(\prod_{i=1}^{l}\norm[2]{W_i}\right)\left(1+\frac{1}{l}\right)^{l}\left[ \frac{\norm[2]{\dww{l}}}{\norm[2]{W_l}}\left(1+\frac{1}{l}\right)^{-l} \right. \nonumber\\
		&\quad \left. + \sum_{k=1}^{l-1}\frac{\norm[2]{\dww{k}}}{\norm[2]{W_{k}}}\left(1+\frac{1}{l}\right)^{-k}\right] \nonumber\\
		&\le \frac{e}{\sqrt{n}}\norm[F]{X}\norm[2]{\pg}^{l-1}\left(\prod_{i=1}^{l}\norm[2]{W_i}\right)\sum_{k=1}^{l}\frac{\norm[2]{\dww{k}}}{\norm[2]{W_k}}, \nonumber
	\end{align}
	where the last inequality holds since $1\le\left(1+1/l\right)^{l}\le e$.
\end{proof}

The second step is to bound the change in the robust margin operator using \cref{lem:pert-f-gcn}.
\begin{lemma}
	\label{lem:pert-gcn}
	Under Assumptions~\ref{assumption:graph} and~\ref{assumption:hidden-width}, for any $w$ and any perturbation $\dw=vec\left(\left\{\dww{i}\right\}_{i=1}^{l}\right)$ such that for all $i\in[l]$, $\norm[2]{\dww{i}}\le\norm[2]{W_i}/l$,	the change in the margin operator is bounded as 
	\begin{align}
		\left|\mar{\fw{w+\dw}}{i,j}-\mar{\fw{w}}{i,j}\right| \le 2eB\left(\prod_{i=1}^{l}\norm[2]{W_i}\right)\sum_{i=1}^{l}\frac{\norm[2]{\dww{i}}}{\norm[2]{W_i}}. \nonumber
	\end{align}
	Further, consider the $\epsilon$-attack ($\epsilon>0$), the change in the robust margin operator is bounded as 
	\begin{align}
		\left|\mar[RM]{\fw{w+\dw}}{i,j}-\mar[RM]{\fw{w}}{i,j}\right| \le 2e(B+\epsilon)\left(\prod_{i=1}^{l}\norm[2]{W_i}\right)\sum_{i=1}^{l}\frac{\norm[2]{\dww{i}}}{\norm[2]{W_i}}. \nonumber
	\end{align}
\end{lemma}
\begin{proof}
	For any $i,j\in[K]$, by the definition of margin operator, we have
	\begin{align}
		\left|\mar{\fw{w+\dw}}{i,j}-\mar{\fw{w}}{i,j}\right| &\le \left|\fw{w+\dw}_i-\fw{w}_i\right| + \left|\fw{w+\dw}_j-\fw{w}_j\right| \nonumber\\
		&\le 2\norm[2]{\fw{w+\dw}-\fw{w}}. \nonumber
	\end{align}
	By \cref{prop:graph-norm,lem:pert-f-gcn}, we have
	\begin{align}
		\left|\mar{\fw{w+\dw}}{i,j}-\mar{\fw{w}}{i,j}\right| &\le 2e\norm[2]{X}\norm[2]{\pg}^{l-1}\left(\prod_{i=1}^{l}\norm[2]{W_i}\right)\sum_{i=1}^{l}\frac{\norm[2]{\dww{i}}}{\norm[2]{W_i}} \nonumber\\
		&\le 2eB\left(\prod_{i=1}^{l}\norm[2]{W_i}\right)\sum_{i=1}^{l}\frac{\norm[2]{\dww{i}}}{\norm[2]{W_i}}. \nonumber
	\end{align}
	Further, let
	\begin{align}
		G(w) &= \mathop{\arg\inf}_{G'\in\delta_w(G)}\left(\fw[G']{w}_i - \fw[G']{w}_j\right) = \mathop{\arg\inf}_{G'\in\delta_w(G)} \mar{\fw[G']{w}}{i,j}, \nonumber
	\end{align}
	and $X(w)$ is the corresponding node features matrix. Denoting $\tilde{w}=w+\dw$, we have
	\begin{align}
		&\quad\, \left|\mar[RM]{\fw{\tilde{w}}}{i,j}-\mar[RM]{\fw{w}}{i,j}\right| \nonumber\\
		&= \left|\inf_{G'\in\delta_{\tilde{w}}(G)}\mar{\fw[G']{\tilde{w}}}{i,j} - \inf_{G'\in\delta_{w}(G)}\mar{\fw[G']{w}}{i,j}\right| \nonumber\\
		&= \left|\mar{\fw[G(\tilde{w})]{\tilde{w}}}{i,j} - \mar{\fw[G(w)]{w}}{i,j}\right| \nonumber\\
		&\le \max\left\{\left|\mar{\fw[G(\tilde{w})]{\tilde{w}}}{i,j} - \mar{\fw[G(\tilde{w})]{w}}{i,j}\right|,\right. \nonumber\\
		&\qquad\qquad \left.\left|\mar{\fw[G(w)]{\tilde{w}}}{i,j} - \mar{\fw[G(w)]{w}}{i,j}\right|\right\} \nonumber\\
		&\le 2e\max\left\{\norm[2]{X(\tilde{w})}, \norm[2]{X(w)}\right\}\left(\prod_{i=1}^{l}\norm[2]{W_i}\right)\sum_{i=1}^{l}\frac{\norm[2]{\dww{i}}}{\norm[2]{W_i}} \nonumber\\
		&\le 2e(B+\epsilon)\left(\prod_{i=1}^{l}\norm[2]{W_i}\right)\sum_{i=1}^{l}\frac{\norm[2]{\dww{i}}}{\norm[2]{W_i}}, \nonumber
	\end{align}
	where the first inequality holds since, by the definition of $G(w)$,
	\begin{align}
		& \mar{\fw[G(\tilde{w})]{\tilde{w}}}{i,j} - \mar{\fw[G(w)]{w}}{i,j}\le \mar{\fw[G(w)]{\tilde{w}}}{i,j} - \mar{\fw[G(w)]{w}}{i,j}, \nonumber\\
		& \mar{\fw[G(w)]{w}}{i,j} - \mar{\fw[G(\tilde{w})]{\tilde{w}}}{i,j} \le \mar{\fw[G(\tilde{w})]{w}}{i,j} - \mar{\fw[G(\tilde{w})]{\tilde{w}}}{i,j}, \nonumber
	\end{align}
	and the last inequality holds from Assumption~\ref{assumption:graph} and the definition of $\epsilon$-attack.
\end{proof}

Finally, the key to prove \cref{thm:adv-gen-gcn} using \cref{lem:pert-kl-bound,lem:pert-gcn}, is to construct suitable prior and posterior distributions satisfying the conditions in the lemmas. 

\begin{proof}[Proof of \cref{thm:adv-gen-gcn}]
	Let $\beta=\left(\prod_{i=1}^{l}\norm[2]{W_i}\right)^{1/l}$. We normalize the weights as $\tilde{W}_i=\beta W_i/\norm[2]{W_i}$. Due to the homogeneity of ReLU, we have $f_w=f_{\tilde{w}}$. We can also verify that $\prod_{i=1}^{l}\norm[2]{W_i} = \prod_{i=1}^{l}\norm[2]{\tilde{W}_i}$ and $\norm[F]{W_i}/\norm[2]{W_i} = \norm[F]{\tilde{W}_i}/\norm[2]{\tilde{W}_i}$, i.e., the terms appear in the bound stay the same after applying the normalization. Therefore, without loss of generality, we assume that the norm is equal across layers, i.e., $\forall i, \norm[2]{W_i}=\beta$.
	
	\noindent Consider the prior distribution $P = \N(0,\sigma^2I)$ and the random perturbation $\dw\sim \N(0,\sigma^2I)$. Note that the $\sigma$ of the prior and the perturbation distributions are the same and will be set according to $\beta$. More precisely, we will set the $\sigma$ based on some approximation $\tilde{\beta}$ of $\beta$ since the prior distribution  $P$ cannot depend on any learned weights directly. The approximation $\tilde{\beta}$ is chosen from a covering set which covers the meaningful range of $\beta$. For now, let us assume that we have a fixed $\tilde{\beta}$ and consider $\beta$ which satisfies $|\beta-\tilde{\beta}|\le\beta/l$. Note that 
	\begin{align}
		|\beta-\tilde{\beta}|\le\frac{1}{l}\beta &\Rightarrow \left(1-\frac{1}{l}\right)\beta\le \tilde{\beta}\le \left(1+\frac{1}{l}\right)\beta \nonumber\\
		&\Rightarrow \left(1-\frac{1}{l}\right)^{l-1} \beta^{l-1} \le \tilde{\beta}^{l-1} \le \left(1+\frac{1}{l}\right)^{l-1} \beta^{l-1} \nonumber\\
		&\Rightarrow \left(1-\frac{1}{l}\right)^{l-1} \beta^{l-1} \le \tilde{\beta}^{l-1} \le \left(1+\frac{1}{l}\right)^{l} \beta^{l-1} \nonumber\\
		&\Rightarrow \frac{1}{e}\beta^{l-1} \le \tilde{\beta}^{l-1} \le e\beta^{l-1}, \nonumber
	\end{align}
	where the last two inequalities hold since $\left(1-1/l\right)^{l-1}\ge1/e$ and $\left(1+1/l\right)^{l}\le e$. Indeed, using inequalities $\ln(x)\le x-1, x>0$ and $\ln(x)\ge 1-1/x,x>0$, we can show that $l\ln\left(1+1/l\right)\le 1$ and $(l-1)\ln\left(1-1/l\right)\ge (l-1)\left(1-\frac{1}{1-\frac{1}{l}}\right)=-1$.
	
	From \citep{tropp2012user}, for $\dww{i}\in\R^{h_{i-1}\times h_{i}}$ and $\dww{i}\sim\N(0,\sigma^2I)$, we have
	\begin{align}
		\Prob\left\{\norm[2]{\dww{i}}\ge t\right\} \le (h_{i-1}+h_{i})\exp\left(-\frac{t^2}{2\sigma^2\max(h_{i-1},h_{i})}\right) \le 2h\exp\left(-\frac{t^2}{2h\sigma^2}\right). \nonumber
	\end{align}
	Taking a union bound, we have
	\begin{align}
		\Prob\left(\bigcap_{i=1}^{l}\left\{\norm[2]{\dww{i}}<t\right\}\right) &= 1-\Prob\left(\bigcup_{i=1}^{l}\left\{\norm[2]{\dww{i}}\ge t\right\}\right) \nonumber\\
		&\ge 1 - \sum_{i=1}^{l}\Prob\left\{\norm[2]{\dww{i}}\ge t\right\} \nonumber\\
		&\ge 1 - 2lh\exp\left(-\frac{t^2}{2h\sigma^2}\right). \nonumber
	\end{align}
	Setting $2lh\exp\left(-\frac{t^2}{2h\sigma^2}\right)=\frac{1}{2}$, we have $t=\sigma\sqrt{2h\log(4lh)}$. This implies that
	\begin{align}
		\Prob\left(\bigcap_{i=1}^{l}\left\{\norm[2]{\dww{i}}<\sigma\sqrt{2h\log(4lh)}\right\}\right) &\ge \frac{1}{2}. \nonumber
	\end{align}
	Plugging this bound into \cref{lem:pert-gcn}, we have that with probability at least $1/2$ over $\dw$,
	\begin{align}
		\left|\mar[R]{\fw{w+\dw}}{i,j} - \mar[R]{\fw{w}}{i,j}\right| &\le 2e(B+\epsilon) \left(\prod_{i=1}^{l}\norm[2]{W_i}\right) \sum_{i=1}^{l}\frac{\norm[2]{\dww{i}}}{\norm[2]{W_i}} \nonumber\\
		&= 2e(B+\epsilon)\beta^{l}\sum_{i=1}^{l}\frac{\norm[2]{\dww{i}}}{\beta} \nonumber\\
		&\le 2e(B+\epsilon)\beta^{l-1}l\sigma\sqrt{2h\log(4lh)} \nonumber\\
		&\le 2e^2(B+\epsilon)\tilde{\beta}^{l-1}l\sigma\sqrt{2h\log(4lh)} \le \frac{\gamma}{2}, \nonumber
	\end{align}
	where we set $\sigma=\frac{\gamma}{4e^2(B+\epsilon)\tilde{\beta}^{l-1}l\sqrt{2h\log(4lh)}}$ to get the last inequality. Note that \cref{lem:pert-gcn} also requires $\norm[2]{\dww{i}}\le \norm[2]{W_i}/l, \forall i\in[l]$. The requirement is satisfied if $\sigma\le \frac{\beta}{l\sqrt{2h\log(4lh)}}$, which in turn can be satisfied if 
	\begin{align}
		\frac{\gamma}{4e^2(B+\epsilon)\tilde{\beta}^{l-1}l\sqrt{2h\log(4lh)}} \le \frac{\beta}{l\sqrt{2h\log(4lh)}}. \label{eq:pf-gcn-bound-1}
	\end{align}
	Note that $\frac{\gamma}{4e(B+\epsilon)}\le \beta^{l}$ is a sufficient condition of \cref{eq:pf-gcn-bound-1}. We will see how to satisfy this condition later.
	
	We now compute the KL term in the PAC-Bayes bound in \cref{lem:pert-kl-bound}. As defined above, for a fixed $w$, the prior distribution $P=\N(0,\sigma^2I)$, the posterior distribution $Q$ is the distribution of $\dw$ shifted by $w$, that is, $Q=\N(w,\sigma^2I)$. Then we have,
	\begin{align}
		\kl{Q} &= \frac{1}{2} w\T\left(\sigma^2I\right)\inv w =
		\frac{\norm[2]{w}^2}{2\sigma^2} \nonumber\\
		&= \frac{16e^4(B+\epsilon)^2 \tilde{\beta}^{2l-2}l^22h\log(4lh)}{2\gamma^2} \sum_{i=1}^{l}\norm[F]{W_i}^2 \nonumber\\
		&\le \cO\left(\frac{(B+\epsilon)^2 \beta^{2l}l^2h\log(lh)}{\gamma^2} \sum_{i=1}^{l}\frac{\norm[F]{W_i}^2}{\beta^2}\right) \nonumber\\
		&\le \cO\left((B+\epsilon)^2 l^2h\log(lh)\frac{\prod_{i=1}^{l}\norm[2]{W_i}^2}{\gamma^2} \sum_{i=1}^{l}\frac{\norm[F]{W_i}^2}{\norm[2]{W_i}^2}\right). \nonumber
	\end{align}
	From \cref{lem:pert-kl-bound}, fixing any $\tilde{\beta}$, with probability $1-\delta$ over $S$, for all $w$ such that $|\beta-\tilde{\beta}|\le\beta/l$, we have,
	\begin{align}
		\ml{R}{0} \le \eml{R}{\gamma} + \cO\left(\sqrt{\frac{(B+\epsilon)^2 l^2 h \log(lh)\prod_{i=1}^{l}\norm[2]{W_i}^2\sum_{i=1}^{l}\frac{\norm[F]{W_i}^2}{\norm[2]{W_i}^2} + \log(m/\delta)}{\gamma^2 m}}\right). \label{eq:pf-gcn-bound-2}
	\end{align}
	
	\noindent Finally, we need to consider multiple choices of $\tilde{\beta}$ so that for any $\beta$, we can bound the generalization error like \cref{eq:pf-gcn-bound-2}. First, we only need to consider values of $\beta$ in the following range,
	\begin{align}
		\left(\frac{\gamma}{2(B+\epsilon)}\right)^{1/l} \le \beta \le \left(\frac{\gamma\sqrt{m}}{2(B+\epsilon) }\right)^{1/l}, \label{eq:pf-gcn-bound-3}
	\end{align}
	since otherwise the bound holds trivially as $\ml{R}{0}\le 1$ by definition. If $\beta^{l}<\frac{\gamma}{2(B+\epsilon)}$, then for any $G$ and any $i, j\in[K]$, with \cref{eq:pf-gcn-pert-bound-1},
	\begin{align}
		\left|\mar[RM]{\fw{w}}{i,j}\right| &= \left|\inf_{G'\in\delta_{w}(G)}\mar{\fw[G']{w}}{i,j}\right| = \left|\mar{\fw[\hat{G}]{w}}{i,j}\right| \nonumber\\
		&= \left|\fw[\hat{G}]{w}_{i} -\fw[\hat{G}]{w}_{j}\right| \le 2\left|\fw[\hat{G}]{w}_{i}\right| \nonumber\\
		&\le 2\norm[2]{\fw[\hat{G}]{w}} = 2\norm[2]{\frac{1}{n}\bm{1}_n H_{l-1}^*W_l} \le 2\norm[2]{\frac{1}{n}\bm{1}_n} \norm[F]{H_{l-1}^*W_l} \nonumber\\
		&\le \frac{2}{\sqrt{n}}\norm[F]{H_{l-1}^*}\norm[2]{W_l} \le 2(B+\epsilon) \prod_{i=1}^{l}\norm[2]{W_i} \nonumber\\
		&= 2(B+\epsilon)\beta^{l} < 2(B+\epsilon) \frac{\gamma}{2(B+\epsilon)} = \gamma, \nonumber
	\end{align}
	where $\hat{G}=\arg\inf_{G'\in\delta_{w}(G)}\mar{\fw[G']{w}}{i,j}$ and $H_{l-1}^*$ is the corresponding node representation.
	Therefore, by the definition, we always have $\eml{R}{\gamma}=1$ when $\beta<\left(\frac{\gamma}{2(B+\epsilon)}\right)^{1/l}$.
	
	\noindent Alternatively, if $\beta^{l}>\frac{\gamma\sqrt{m}}{2(B+\epsilon)}$, the term inside the big-O notation in \cref{eq:pf-gcn-bound-2} would be
	\begin{align}
		\sqrt{\frac{(B+\epsilon)^2 l^2 h \log(lh)\prod\limits_{i=1}^{l}\norm[2]{W_i}^2\sum\limits_{i=1}^{l}\dfrac{\norm[F]{W_i}^2}{\norm[2]{W_i}^2} + \log(m/\delta)}{\gamma^2 m}} &\ge \sqrt{\frac{l^2h\log(lh)}{4}\sum_{i=1}^{l}\dfrac{\norm[F]{W_i}^2}{\norm[2]{W_i}^2}} \nonumber\\
		&\ge \sqrt{\frac{l^2h\log(lh)}{4}} \ge 1, \nonumber
	\end{align}	
	where we use the facts that $\norm[F]{W_i}\ge\norm[2]{W_i}$ and we typically choose $h\ge2$ and $l\ge2$ in practice. Note that the lower bound in \cref{eq:pf-gcn-bound-3} ensures that \cref{eq:pf-gcn-bound-1} holds which in turn justifies the applicability of \cref{lem:pert-gcn}.
	\hspace*{\fill}\\
	
	Consider a covering set centered at $C=\left\{\tilde{\beta}_1,\dots,\tilde{\beta}_{|C|}\right\}$ of the interval in \cref{eq:pf-gcn-bound-3} with radius $\frac{1}{l}\left(\frac{\gamma}{2(B+\epsilon)}\right)^{1/l}$, where $|C|$ is the size of the set. To make sure that $|\beta-\tilde{\beta}|\le\left(\frac{\gamma}{2(B+\epsilon)}\right)^{1/l}/l\le\beta/l$, the minimal size of $C$ should be
	\begin{align}
		|C| = \frac{\left(\frac{\gamma\sqrt{m}}{2(B+\epsilon) }\right)^{1/l} - \left(\frac{\gamma}{2(B+\epsilon)}\right)^{1/l}}{\frac{2}{l}\left(\frac{\gamma}{2(B+\epsilon)}\right)^{1/l}} = \frac{l}{2}\left(m^{\frac{1}{2l}}-1\right). \nonumber
	\end{align}
	
	For a $\tilde{\beta}_i\in C$, we denote the event that \cref{eq:pf-gcn-bound-2} holds for $|\beta-\tilde{\beta}|\le\beta/l$ as $E_i$. Therefore, the probability of the event that \cref{eq:pf-gcn-bound-2} holds for any $\beta$ satisfying \cref{eq:pf-gcn-bound-3} is
	\begin{align}
		\Prob\left(\bigcap_{i=1}^{|C|}E_i\right) = 1-\Prob\left(\bigcup_{i=1}^{|C|}\bar{E}_i\right) \ge 1-\sum_{i=1}^{|C|}\Prob\left\{\bar{E}_i\right\} \ge 1-|C|\delta, \nonumber
	\end{align}
	where $\bar{E}_i$ denotes the complement of $E_i$. Hence, let $\delta'=|C|\delta$, then we have that, with probability $1-\delta'$ and for all $w$,
	\begin{align}
		\ml{L}{0} &\le \eml{L}{\gamma} + \cO\left(\sqrt{\frac{(B+\epsilon)^2 l^2 h \log(lh)\prod_{i=1}^{l}\norm[2]{W_i}^2\sum_{i=1}^{l}\frac{\norm[F]{W_i}^2}{\norm[2]{W_i}^2} + \log(m/\delta)}{\gamma^2 m}}\right) \nonumber\\
		&= \eml{L}{\gamma} + \cO\left(\sqrt{\frac{(B+\epsilon)^2 l^2 h \log(lh)\prod_{i=1}^{l}\norm[2]{W_i}^2\sum_{i=1}^{l}\frac{\norm[F]{W_i}^2}{\norm[2]{W_i}^2} + \log(ml/\delta')}{\gamma^2 m}}\right). \nonumber
	\end{align}

\end{proof}

\subsection{Generalization Bounds for MPGNN}
In this part, we provide the adversarially robust generalization bounds for MPGNN.

\begin{theorem}[Robust Generalization Bounds for MPGNN]
	\label{thm:adv-gen-mpgnn}
	Under Assumptions~\ref{assumption:graph} to~\ref{assumption:mpgnn}, consider the $\epsilon$-attack ($\epsilon>0$). For any $\delta,\gamma>0$, with probability at least $1-\delta$ over the choice of a training set $S=\{z_1,\dots,z_m\}$ independently sampled from $\D$, for any $w$, the following results hold: \\ if $\tau=1$, then
	\begin{align}
		& \ml{R}{0} \le \eml{R}{\gamma} + \cO\left(\sqrt{\frac{(l+1)^4 \Psi_0\Psi_1(f_w) + \log(m/\delta)}{\gamma^2 m}}\right), \nonumber
	\end{align}
	if $\tau\ne1$, then
	\begin{align}
		& \ml{R}{0} \le \eml{R}{\gamma} + \cO\left(\sqrt{\frac{l^2 \Psi_0\Psi_2(f_w) + \log(m(l+1)/\delta)}{\gamma^2 m}}\right). \nonumber
	\end{align}
	Here, $\tau=L^3M_2\max_{\D}\{\norm[2]{\pg}\}$, $\xi=LM_1M_2\frac{\tau^{l-1}-1}{\tau-1},$
	$\zeta=\min\left(\left\{\norm[2]{W_k}\right\}_{k=1}^{l}, \left\{\norm[2]{U_k}\right\}_{k=1}^{l-1}\right)$, $|w|_2^2=\sum_{i=1}^{l}\norm[F]{W_i}^2+\sum_{i=1}^{l-1}\norm[F]{U_i}^2$, $\Psi_0 = (B+\epsilon)^2 h \log(lh),$
	\begin{align}
		\Psi_1(f_w) &= \max\left\{\zeta^{-6},(LM_1M_2)^{3}\right\}|w|_2^2, \nonumber\\
		\Psi_2(f_w) &= \max\left\{\zeta\inv,\xi^{1/l}\right\}^{2l+2}|w|_2^2. \nonumber
	\end{align}
\end{theorem}
In the above results, the bounds vary for two typical settings of $\tau$. The generalization bound for $\tau \neq 1$ scales as $\cO\left(\norm[2]{\pg}^{l-2}l\sqrt{h\log(lh)}\right)$. In practice, $\pg$ is usually chosen as the adjacency matrix $A$ of the graph, where  $\|\pg\|$ is less than the maximum degree $d$ of the graph.

\subsection{Comparison} \label{sc:comparison}


We compare our results with some existing bounds in the standard settings.
For generalization bounds from both models, we will focus on the 
dependency on the depth $l$ and width $h$ as well as maximum degree $d$ of the graph. 

For GCN in standard settings, 
\citet{liao_pac-bayesian_2020} derive PAC-Bayesian generalization bounds of order $\cO\left(d^{\frac{l-1}{2}}l\sqrt{h \log(lh)}\right)$, under exactly the same set-up and assumptions in \cref{thm:gen-gcn} of this paper. Our  derived bounds from \cref{thm:gen-gcn} scale as $\cO\left(l \sqrt{h \log(lh)}\right)$. The latter is tighter than the former with a factor of the maximum degree $d$, which is usually much larger than $1$ in practice. 



For MPGNN in standard settings, the Rademacher complexity bound in \citep{garg_generalization_2020} scales as $\cO\left(h\sqrt{l\log(hd)}\right)$. For PAC-Bayesian framework, the bound in \citep{liao_pac-bayesian_2020} scales as $\cO\left(d^{l-2}l\sqrt{h\log(lh)}\right)$, and the bound in \citep{ju_generalization_2023} scales as $\cO\left(\norm[2]{\pg}^{l-1}l\sqrt{h}\right)$, while our bound scales as $\cO\left(\norm[2]{\pg}^{l-2}l\sqrt{h\log(lh)}\right)$. Note that \citep{ju_generalization_2023} requires  that the nonlinear activations and the loss functions are twice-differentiable, Lipschitz-continuous and their first-order
and second-order derivatives are both Lipschitz-continuous, which thus does not apply to the commonly used ReLU activation. 
When $\pg$ is chosen as the adjacency matrix $A$ of the graph, 
it has been shown that for any undirected graph, the spectral norm of $\pg$ is less than the maximum degree $d$. See \cref{pf:basic-results} for a proof.

We also compare our robust bounds with the bounds on DNN using PAC-Bayesian framework as presented in \citep{xiao_pac-bayesian_nodate}. Given that GCNs are essentially a specialized DNNs with a similar structure, our robust bounds align well with those in \citep{xiao_pac-bayesian_nodate}. 
For MPGNN, the robust bound for $\tau \neq 1$ scales as $\cO\left(\norm[2]{\pg}^{l-2}l\sqrt{h\log(lh)}\right)$, where $\norm{\pg}$ is less than or equal to $d$. Compared to DNNs, MPGNNs exhibit a more complex structure, and the graph structure significantly influences their performance. Consequently, our results include the additional term $d^l$, which is not present in the bounds for DNNs. 

\section{Discussion} \label{sc:discussion}

In this paper, we present adversarially robust generalization bounds for two kinds of GNNs, GCN and MPGNN. Our results in adversarial settings complement the bounds in standard settings. And our results for GCN improve the bounds in \citep{liao_pac-bayesian_2020} by avoiding the dependency on the maximum degree.

Our results are limited to graph classification problem, and cannot be generalized to node classification problem, in which the samples are not independent. We use Gaussian prior and posterior distributions in the analysis, while alternative
distributions could be considered. Additionally, in the analysis, we do not consider the optimization algorithm  that has an impact on the leaned parameters.
This paper leads to a few interesting  problems for future work: (i) Would the analysis still work for other kinds of GNNs, or  other kinds of networks? (ii) How the optimization algorithms like SGD impact the generalization of GNNs?


\section*{Acknowledgments}
This work was supported in part by the National Key Research and Development Program of China under grant number 2021YFA1003500.

\appendix
\section{Proofs} \label{sc:proofs}

\subsection{Basic Results} \label{pf:basic-results}
\begin{proposition}
	\label{prop:f-norm}
	For any matrix $A\in\R^{n\times m}, B\in\R^{m\times p}$, we have, 
	\begin{align}
		\norm[F]{AB} \le \norm[F]{A}\norm[2]{B}. \nonumber
	\end{align}
\end{proposition}
\begin{proof}
	Let $x_i\T, a_i\T, i\in[n]$ be the $i$-th row of $AB$ and $A$, respectively, then we have,
	\begin{align}
		\norm[F]{AB}^{2} &= \sum_{i=1}^{n}\norm[2]{x_i\T}^{2} = \sum_{i=1}^{n}\norm[2]{a_i\T B}^{2} \le \sum_{i=1}^{n}\norm[2]{a_i\T}^{2}\norm[2]{B}^{2} = \norm[F]{A}^{2}\norm[2]{B}^{2}. \qedhere \nonumber
	\end{align}
\end{proof}

\begin{proposition}
	\label{prop:graph-norm}
	For any undirected graph $G=(V,E)$ with $n$ nodes, let $A\in\R^{n\times n}$ be the adjacency matrix, $D=\mathrm{diag}(D_{1},\dots,D_{n})$ be the degree matrix, and $d=\max_{i\in[n]}\left\{D_{i}\right\}$ be the maximum degree, where $D_{i}=\sum_{j=1}^{n}A_{ij},i\in[n]$. Then we have,
	\begin{enumerate}[\quad(i)]
		\item $\norm[2]{A} \le d$;
		\item $\norm[2]{\tilde{D}^{-1/2}\tilde{A}\tilde{D}^{-1/2}} \le 1$.
	\end{enumerate}
	Here $\tilde{A}=A+I$ and $\tilde{D}=\mathrm{diag}\left(\tilde{D}_{1},\dots,\tilde{D}_{n}\right)=\mathrm{diag}\left(\sum_{j=1}^{n}\tilde{A}_{ij}, i\in[n]\right)$.
\end{proposition}
\begin{proof}
	For (i), by the definition of spectral norm, we have
	\begin{align}
		\norm[2]{A} = \max_{\norm[2]{x}=1}x\T Ax = \max_{\norm[2]{x}=1}\sum_{(i,j)\in E}x_i x_j \le \max_{\norm[2]{x}=1}\sum_{(i,j)\in E}\frac{x_i^2+x_j^2}{2} \le \max_{\norm[2]{x}=1} d\sum_{i\in V}x_i^2 = d. \nonumber
	\end{align}
	For (ii), let $\tilde{E}=E\cup\{(i,i)|i\in V\}$ be the edge set associated with the adjacency matrix $\tilde{A}$. By the definition of spectral norm, we have
	\begin{align}
		\norm[2]{\tilde{D}^{-1/2}\tilde{A}\tilde{D}^{-1/2}} &= \max_{\norm[2]{x}=1}x\T \left(\tilde{D}^{-1/2}\tilde{A}\tilde{D}^{-1/2}\right)x = \max_{\norm[2]{x}=1}\sum_{(i,j)\in \tilde{E}}\frac{x_i x_j}{\sqrt{\tilde{D}_i \tilde{D}_j}} \nonumber\\
		&\le \max_{\norm[2]{x}=1}\sum_{(i,j)\in \tilde{E}}\left(\frac{x_i^2}{2\tilde{D}_i} + \frac{x_j^2}{2\tilde{D}_j}\right) \le \max_{\norm[2]{x}=1} \sum_{i\in V}x_i^2 = 1. \qedhere \nonumber
	\end{align}
\end{proof}

\subsection{Results of PAC-Bayesian Analysis}
\subsubsection{Proof of \cref{thm:mcallester}} \label{pf:thm-mcallester}
Following \citep{xiao_pac-bayesian_nodate}, we provide the proof for the PAC-Bayesian results with the robust margin loss.  First, we need the following lemma.
\begin{lemma}
	\label{lem:2-side}
	Let $X$ be a random variable satisfying $\Prob(X\ge\epsilon)\le e^{-2m\epsilon^{2}}$ and $\Prob(X\le-\epsilon)\le e^{-2m\epsilon^{2}}$, where $m\ge1$ and $\epsilon>0$, we have
	\begin{align}
		\E\left[e^{2(m-1)X^2}\right] \le 2m. \nonumber
	\end{align}
\end{lemma}
\begin{proof}
	If $m=1$, the inequality holds trivially. We consider $m>1$. By the definition of expectation, we have
	\begin{align}
		\E\left[e^{2(m-1)X^2}\right] &= \int_{0}^{\infty}\Prob\left(e^{2(m-1)X^2}\ge u\right)\id u \nonumber\\
		&= \int_{0}^{\infty}\Prob\left(X^2\ge \frac{\log u}{2(m-1)}\right)\id u \nonumber\\
		&= \int_{0}^{\infty}\Prob\left(X\ge \sqrt{\frac{\log u}{2(m-1)}}\right)\id u + \int_{0}^{\infty}\Prob\left(X\le -\sqrt{\frac{\log u}{2(m-1)}}\right)\id u. \nonumber
	\end{align}
	Then,
	\begin{align}
		\int_{0}^{\infty}\Prob\left(X\ge \sqrt{\frac{\log u}{2(m-1)}}\right)\id u 
		&= \int_{0}^{1}\Prob\left(X\ge \sqrt{\frac{\log u}{2(m-1)}}\right)\id u \nonumber\\
		&\quad + \int_{1}^{\infty}\Prob\left(X\ge \sqrt{\frac{\log u}{2(m-1)}}\right)\id u \nonumber\\
		&\le 1 + \int_{1}^{\infty}\Prob\left(X\ge \sqrt{\frac{\log u}{2(m-1)}}\right)\id u \nonumber\\
		&\le 1 + \int_{1}^{\infty}e^{-2m\frac{\log u}{2(m-1)}}\id u \nonumber\\
		&= 1 + \left(-(m-1)u^{-\frac{1}{m-1}}\Big|_{1}^{\infty}\right) \nonumber\\
		&= m. \nonumber
	\end{align}
	Similarly, we have
	\begin{align}
		\int_{0}^{\infty}\Prob\left(X\le -\sqrt{\frac{\log u}{2(m-1)}}\right)\id u &\le m. \nonumber
	\end{align}
	Combing the above analysis, we obtain the desired result. 
\end{proof}

\begin{proof}[Proof of \cref{thm:mcallester}]
	For any function $u(h)$ defined on $\cH$, we have
	\begin{align}
		\E_{h\sim Q}\left[u(h)\right] &= \E_{h\sim Q}\left[\log e^{u(h)}\right] \nonumber\\
		&= \E_{h\sim Q}\left[\log e^{u(h)} + \log\frac{Q}{P} + \log\frac{P}{Q}\right] \nonumber\\
		&= \kl{Q} + \E_{h\sim Q}\left[\log \left(\frac{P}{Q}e^{u(h)}\right)\right] \nonumber\\
		&\le \kl{Q} + \log\E_{h\sim Q}\left[\frac{P}{Q}e^{u(h)}\right] \qquad(\mbox{Jensen's inequality}) \nonumber\\
		&= \kl{Q} + \log\E_{h\sim P}\left[e^{u(h)}\right]. \nonumber
	\end{align}
	For $h\in\cH$, denote $\Delta(h) = \ml[h]{R}{\gamma} - \eml[h]{R}{\gamma}$ and $u(h)=2(m-1)\Delta(h)^2$. We have
	\begin{align}
		2(m-1)\left[\E_{h\sim Q}\left[\Delta(h)\right]\right]^2 &\le 2(m-1)\E_{h\sim Q}\left[\Delta(h)^2\right] \qquad(\mbox{Jensen's inequality}) \nonumber\\
		&\le \kl{Q} + \log\E_{h\sim P}\left[e^{2(m-1)\Delta(h)^2}\right]. \label{eq:pf-mcallester-two-side-1}
	\end{align}
	Since $\ml[h]{R}{\gamma}\in[0,1]$, based on Hoeffding's inequality, for any $\epsilon>0$, we have
	\begin{align}
		\Prob_{S}\left(\Delta(h)\ge \epsilon\right) \le e^{-2m\epsilon^2}, \nonumber\\
		\Prob_{S}\left(\Delta(h)\le -\epsilon\right) \le e^{-2m\epsilon^2}. \nonumber
	\end{align}
	Hence, based on \cref{lem:2-side}, we have
	\begin{align}
		\E_{S} \left[e^{2(m-1)\Delta(h)^2}\right]\le 2m\ &\Rightarrow\ \E_{h\sim P}\left[\E_{S}\left[e^{2(m-1)\Delta(h)^2}\right]\right] \le 2m \nonumber\\
		&\Leftrightarrow\ \E_{S}\left[\E_{h\sim P}\left[e^{2(m-1)\Delta(h)^2}\right]\right] \le 2m. \nonumber
	\end{align}
	Based on Markov's inequality, we have
	\begin{align}
		\Prob_{S}\left(\E_{h\sim P}\left[e^{2(m-1)\Delta(h)^2}\right]\ge\frac{2m}{\delta}\right) \le \frac{\delta\E_{S}\left[\E_{h\sim P}\left[e^{2(m-1)\Delta(h)^2}\right]\right]}{2m} \le \delta. \label{eq:pf-mcallester-two-side-2}
	\end{align}
	Combining \cref{eq:pf-mcallester-two-side-1,eq:pf-mcallester-two-side-2}, we have that with probability at least $1-\delta$ over $S$, 
	\begin{align}
		\ml[Q]{R}{\gamma} - \eml[Q]{R}{\gamma} &= \E_{h\sim Q}\left[\ml[h]{R}{\gamma} - \eml[h]{R}{\gamma}\right] \nonumber\\
		&\le \sqrt{\frac{\kl{Q} + \log\E_{h\sim P}\left[e^{2(m-1)\Delta(h)^2}\right]}{2(m-1)}} \nonumber\\
		&\le \sqrt{\frac{\kl{Q} + \log(2m/\delta)}{2(m-1)}}. \qedhere \nonumber
	\end{align}
\end{proof}

\subsubsection{Proof of \cref{lem:pert-kl-bound}} \label{pf:lem-pert-kl-bound}
\begin{proof}[Proof of \cref{lem:pert-kl-bound}]
	For any $w$ and random perturbation $\dw$, let $w'=w+\dw$. Let
	\begin{align}
		\mathcal{S}_{w} = \left\{w' \mid \max_{i,j\in[k],\atop G\in\X\times\G} \left|\mar[RM]{\fw{w'}}{i,j} - \mar[RM]{\fw{w}}{i,j}\right|\le \frac{\gamma}{2}\right\}. \nonumber
	\end{align}
	Let $q$ be the probability density function over the parameters $w'$. We construct a new distribution $\tilde{Q}$ over predictors $f_{\tilde{w}}$, where $\tilde{w}$ is restricted to $\sw$ with the probability density function
	\begin{align}
		\tilde{q}(\tilde{w}) = \frac{1}{Z}\begin{cases}
			q(\tilde{w}), & \tilde{w}\in\sw, \\ 0, & \tilde{w}\in\sw^c.
		\end{cases} \nonumber
	\end{align}
	Here $Z=\Prob\left\{w'\in\sw\right\}\ge\frac{1}{2}$ and $\sw^c$ is the complement set of $\sw$. By the definition of $\tilde{Q}$, for $\tilde{w}\sim\tilde{Q}$, we have
	\begin{align}
		\max_{i,j\in[k],\atop G\in\X\times\G} \left|\mar[RM]{\fw{\tilde{w}}}{i,j} - \mar[RM]{\fw{w}}{i,j}\right|\le \frac{\gamma}{2}. \label{eq:ppf-bound-lm-1}
	\end{align}
	Recall that
	\begin{align}
		\ml{R}{0} &= \Prob_{(G,y)\sim\D} \left\{\mar[RM]{\fw{w}}{y} \le 0\right\}, \nonumber
	\end{align}
	and 
	\begin{align}
		\ml[f_{\tilde{w}}]{R}{\frac{\gamma}{2}} &= \Prob_{(G,y)\sim\D} \left\{\mar[RM]{\fw{\tilde{w}}}{y} \le \frac{\gamma}{2}\right\}. \nonumber
	\end{align}
	From \cref{eq:ppf-bound-lm-1}, we have
	\begin{align}
		\mar[RM]{\fw{\tilde{w}}}{y,j} \le \mar[RM]{\fw{w}}{y,j} + \frac{\gamma}{2}. \nonumber
	\end{align}
	Then, by the definition of robust margin operator, we have
	\begin{align}
		\mar[RM]{\fw{\tilde{w}}}{y} &= \inf_{G'\in\delta_{\tilde{w}}(G)}\left(\fw[G']{\tilde{w}}_y - \max_{j\ne y}\fw[G']{\tilde{w}}_j\right) \nonumber\\
		&= \inf_{G'\in\delta_{\tilde{w}}(G)}\inf_{j\ne y}\left(\fw[G']{\tilde{w}}_y - \fw[G']{\tilde{w}}_j\right) \nonumber\\
		&= \inf_{j\ne y}\inf_{G'\in\delta_{\tilde{w}}(G)}\left(\fw[G']{\tilde{w}}_y - \fw[G']{\tilde{w}}_j\right) \nonumber\\
		&= \inf_{j\ne y}\mar[RM]{\fw{\tilde{w}}}{y,j} \nonumber\\
		&\le \inf_{j\ne y}\mar[RM]{\fw{w}}{y,j} + \frac{\gamma}{2} \nonumber\\
		&= \mar[RM]{\fw{w}}{y} + \frac{\gamma}{2}. \nonumber
	\end{align}
	Therefore, we have 
	\begin{align}
		\mar[RM]{\fw{w}}{y} \le 0\ \Rightarrow\ \mar[RM]{\fw{\tilde{w}}}{y} \le \frac{\gamma}{2} . \nonumber
	\end{align}
	Since it holds for any sample $(G,y)\in\X\times\G\times\Y$, we can get
	\begin{align}\label{eq:ppf-bound-lm-2}
		\ml{R}{0} \le \ml[f_{\tilde{w}}]{R}{\frac{\gamma}{2}}.
	\end{align}
	
	Similarly, recall that 
	\begin{align}
		\eml{R}{\frac{\gamma}{2}} = \frac{1}{m}\sum_{(G,y)\in S}\mathbbm{1}\left\{\mar[RM]{\fw{\tilde{w}}}{y}\le \frac{\gamma}{2}\right\}, \nonumber
	\end{align}
	and 
	\begin{align}
		\eml{R}{\gamma} = \frac{1}{m}\sum_{(G,y)\in S}\mathbbm{1}\left\{\mar[RM]{\fw{w}}{y}\le \gamma\right\}. \nonumber
	\end{align}
	From \cref{eq:ppf-bound-lm-1}, we have
	\begin{align}
		\mar[RM]{\fw{w}}{y,j} \le \mar[RM]{\fw{\tilde{w}}}{y,j} + \frac{\gamma}{2}. \nonumber
	\end{align}
	Then, by the definition of robust margin operator, we have
	\begin{align}
		\mar[RM]{\fw{w}}{y} &= \inf_{G'\in\delta_{w}(G)}\left(\fw[G']{w}_y - \max_{j\ne y}\fw[G']{w}_j\right) \nonumber\\
		&= \inf_{G'\in\delta_{w}(G)}\inf_{j\ne y}\left(\fw[G']{w}_y - \fw[G']{w}_j\right) \nonumber\\
		&= \inf_{j\ne y}\inf_{G'\in\delta_{w}(G)}\left(\fw[G']{w}_y - \fw[G']{w}_j\right) \nonumber\\
		&= \inf_{j\ne y}\mar[RM]{\fw{w}}{y,j} \nonumber\\
		&\le \inf_{j\ne y}\mar[RM]{\fw{\tilde{w}}}{y,j} + \frac{\gamma}{2} \nonumber\\
		&= \mar[RM]{\fw{\tilde{w}}}{y} + \frac{\gamma}{2}. \nonumber
	\end{align}
	Therefore, we have 
	\begin{align}
		\mar[RM]{\fw{\tilde{w}}}{y} \le \frac{\gamma}{2}\ \Rightarrow\ \mar[RM]{\fw{w}}{y} \le \gamma, \nonumber
	\end{align}
	which indicates that 
	\begin{align}
		\mathbbm{1}\left\{\mar[RM]{\fw{\tilde{w}}}{y}\le \frac{\gamma}{2}\right\} \le \mathbbm{1}\left\{\mar[RM]{\fw{w}}{y}\le\gamma\right\}. \nonumber
	\end{align}
	Since it holds for any sample $(G,y)\in\X\times\G\times\Y$, we can get
	\begin{align}
		\eml[f_{\tilde{w}}]{R}{\frac{\gamma}{2}} \le \eml{R}{\gamma}. \label{eq:ppf-bound-lm-3}
	\end{align}
	Now, with probability at least $1-\delta$ over the choice of a training set $S=\{z_1,\dots,z_m\}$ independently sampled from $\D$, we have,
	\begin{align}
		\ml{R}{0} &\le  \E_{\tilde{w}}\left[\ml[f_{\tilde{w}}]{R}{\frac{\gamma}{2}}\right] &(\mbox{\cref{eq:ppf-bound-lm-2}}) \nonumber\\
		&\le \E_{\tilde{w}}\left[\eml[f_{\tilde{w}}]{R}{\frac{\gamma}{2}}\right] + \sqrt{\frac{KL(\tilde{Q}\Vert P)+\log(2m/\delta)}{2(m-1)}} &(\mbox{\cref{thm:mcallester}})\nonumber\\
		&\le \eml[f_{w}]{R}{\gamma} + \sqrt{\frac{KL(\tilde{Q}\Vert P)+\log(2m/\delta)}{2(m-1)}}. &(\mbox{\cref{eq:ppf-bound-lm-3}}) \nonumber
	\end{align}
	Note that
	\begin{align}
		\kl{Q} &= \int_{\tilde{w}\in\sw}Q\log\frac{Q}{P}d\tilde{w} + \int_{\tilde{w}\in\sw^c}Q\log\frac{Q}{P}d\tilde{w} \nonumber\\
		&= \int_{\tilde{w}\in\sw}\frac{QZ}{Z}\log\frac{QZ}{PZ}d\tilde{w} + \int_{\tilde{w}\in\sw^c}\frac{Q(1-Z)}{1-Z}\log\frac{Q(1-Z)}{P(1-Z)}d\tilde{w} \nonumber\\
		&= Z\int_{\tilde{w}\in\sw}\frac{Q}{Z}\log\frac{Q}{ZP}d\tilde{w} + \int_{\tilde{w}\in\sw}Q\log Zd\tilde{w} \nonumber\\
		&\ \ +  (1-Z)\int_{\tilde{w}\in\sw^c}\frac{Q}{1-Z}\log\frac{Q}{(1-Z)P}d\tilde{w} + \int_{\tilde{w}\in\sw^c}Q\log(1-Z)d\tilde{w} \nonumber\\
		&= Z\kl{\tilde{Q}} + (1-Z)\kl{\bar{Q}} - H(Z), \nonumber
	\end{align}
	where $\bar{Q}=\frac{Q}{1-Z}$ and $H(Z)=-Z\log Z-(1-Z)\log(1-Z)$. Since $H(z)$ is non-increasing on $[\frac{1}{2},1)$, we have $0\le H(Z)\le\log(2)$. Since \textit{KL} is always positive, we have
	\begin{align}
		\kl{\tilde{Q}} &= \frac{1}{Z}\left[\kl{Q}-(1-Z)\kl{\bar{Q}} + H(Z)\right]\nonumber\\
		&\le 2\kl{Q} + 2\log(2). \nonumber
	\end{align}
	Hence, with probability at least $1-\delta$ over the choice of a training set $S=\{z_1,\dots,z_m\}$ independently sampled from $\D$,
	\begin{align}
		\ml{R}{0} &\le \eml[f_{{w}}]{R}{\gamma} + \sqrt{\frac{KL(\tilde{Q}\Vert P)+\log(2m/\delta)}{2(m-1)}} \nonumber\\
		&\le \eml[f_{{w}}]{R}{\gamma} + \sqrt{\frac{2\kl{Q}+2\log 2+\log(2m/\delta)}{2(m-1)}} \nonumber\\
		&= \eml[f_{{w}}]{R}{\gamma} + \sqrt{\frac{2\kl{Q}+\log(8m/\delta)}{2(m-1)}}. \qedhere\nonumber
	\end{align}
\end{proof}

\subsection{Results of MPGNN}\label{pf:mpgnn}
\subsubsection{Proof of Perturbation Bound for MPGNN}
\label{pf:lem-pert-operator-mpgnn}

\begin{lemma}
	\label{lem:pert-f-mpgnn}
	Assume that for any MPGNN model $f_w\in\cH$ with $l$ layers. 
	For any $w$ such that $\norm[2]{U_i}\le M_1, \forall i\in[l-1], \norm[2]{W_i}\le M_2, \forall i\in[l]$ and any perturbation $\dw=vec\left(\left\{\dww{i}\right\}_{i=1}^{l}\right)$ such that $\eta=\max\left(\left\{\frac{\norm[2]{\duu{i}}}{\norm[2]{U_i}}\right\}_{i=1}^{l-1},\left\{\frac{\norm[2]{\dww{i}}}{\norm[2]{W_i}}\right\}_{i=1}^{l}\right)\le {1\over l}, \forall i\in[l]$, for any $G=(X,A)\in\X\times\G$, the change in the output of MPGNN is bounded as 
	\begin{align}
		\norm[2]{\fw{w+\dw} - \fw{w}} &\le \begin{cases}
			e\eta LM_1M_2\norm[2]{X}(l+1)^2, & \tau=1,\\
			e\eta lLM_1M_2\norm[2]{X}\dfrac{\tau^{l-1}-1}{\tau-1}, & \tau\ne1,
		\end{cases} \nonumber
	\end{align}
	where $\tau=L^3M_2\norm[2]{\pg}$.
\end{lemma}
\begin{proof}
	We denote the node representation in $j$-th layer as
	\begin{align}
		& f_{w}^{j}(G) = H_{j} = \phi_j\left(E_{j}\right), \ E_{j} = XU_{j} + \rho_j\left(F_{j}\right), \ F_{j} = \pg\psi_j\left(H_{j-1}\right)W_{j}, \ j\in[l-1], \nonumber\\
		& f_{w}^{l}(G) = H_l = \frac{1}{n}\bone H_{l-1}W_l. \nonumber \nonumber
	\end{align}
	Adding perturbation $\dw$ to the weights $w$, i.e., for the $j$-th layer, the perturbed weights are $W_j+\dww{j}$ or $U_j+\duu{j}$ and denote $H_{j}',\ E_j',\ F_j'$ as the perturbed term respectively.
	
	\textbf{Upper Bound on the Node Representation. }For any $j<l$,
	\begin{align}
		\norm[F]{F_j} &= \norm[F]{\pg\psi_j\left(H_{j-1}\right)W_{j}} \nonumber\\
		&\le \norm[2]{\pg}\norm[F]{\psi_j\left(H_{j-1}\right)}\norm[2]{W_{j}} \nonumber\\
		&\le L M_2\norm[2]{\pg}\norm[F]{H_{j-1}}, \nonumber
	\end{align}
	where the first inequality holds by \cref{prop:f-norm}, and the last one holds by the assumption, the Lipschitzness of $\psi_j$ with $\psi_j(0)=0$. Then we have, for any $j<l$,
	\begin{align}
		\norm[F]{H_{j}} &= \norm[F]{\phi_j\left(E_j\right)} \le L\norm[F]{XU_{j} + \rho_j\left(F_j\right)} \nonumber\\
		&\le L\norm[F]{XU_{j}} + L\norm[F]{\rho_j\left(F_j\right)} \nonumber\\
		&\le L\norm[F]{X}\norm[2]{U_{j}} + L^{2}\norm[F]{F_j} \nonumber\\
		&\le L M_1 \norm[F]{X} + L^{3}M_2\norm[2]{\pg}\norm[F]{H_{j-1}}. \nonumber
	\end{align}
	
	Denoting $\kappa=LM_1\norm[F]{X}, \tau=L^{3}M_2\norm[2]{\pg}$, by applying the relationship iteratively and $H_0=0$, we have, for $j<l$,
	\begin{align}
		\norm[F]{H_{j}} &\le \kappa + \tau\norm[F]{H_{j-1}} \nonumber\\
		&\le \tau^{j}\norm[F]{H_{0}} + \sum_{i=0}^{j-1}\kappa \tau^{j-1-i} = \sum_{i=0}^{j-1}\kappa \tau^{j-1-i} \nonumber\\
		&= \begin{cases}
			\kappa j, & \tau = 1, \\
			\kappa\dfrac{\tau^{j}-1}{\tau-1}, & \tau\ne 1. \\
		\end{cases} \label{eq:pf-lem-pert-f-mpgnn-1}
	\end{align}
	
	
	\textbf{Upper Bound on the Change of Node Representation. }For any $j<l$,
	\begin{align}
		\norm[F]{F_{j}'-F_{j}} &= \norm[F]{\pg\psi_j\left(H_{j-1}'\right)\left(W_{j}+\dww{j}\right) - \pg\psi_j\left(H_{j-1}\right)W_{j}} \nonumber\\
		&= \norm[F]{\pg\left[\psi_j\left(H_{j-1}'\right)-\psi_j\left(H_{j-1}\right)\right]\left(W_{j}+\dww{j}\right) + \pg\psi_j\left(H_{j-1}\right)\dww{j}} \nonumber\\
		&\le \norm[F]{\pg\left[\psi_j\left(H_{j-1}'\right)-\psi_j\left(H_{j-1}\right)\right]\left(W_{j}+\dww{j}\right)} + \norm[F]{\pg\psi_j\left(H_{j-1}\right)\dww{j}} \nonumber\\
		&\le \norm[2]{\pg}\norm[F]{\psi_j\left(H_{j-1}'\right)-\psi_j\left(H_{j-1}\right)}\norm[2]{W_{j}+\dww{j}} + \norm[2]{\pg}\norm[F]{\psi_j\left(H_{j-1}\right)}\norm[2]{\dww{j}} \nonumber\\
		&\le L \left(1+\eta\right)M_2 \norm[2]{\pg}\norm[F]{H_{j-1}'-H_{j-1}} + L\eta M_2 \norm[2]{\pg}\norm[F]{H_{j-1}}, \nonumber
	\end{align}
	where the second inequality holds from \cref{prop:f-norm} and the last one holds by the assumption and the Lipschitzness of $\psi_j$ with $\psi_j(0)=0$. Then we have, for any $j<l$,
	\begin{align}
		\norm[F]{H_{j}'-H_{j}} &= \norm[F]{\phi_j(E_j')-\phi_j(E_j)} \nonumber\\
		&\le L \norm[F]{\left[X\left(U_{j}+\duu{j}\right) + \rho_j\left(F_j'\right)\right] - \left[XU_{j} + \rho_j\left(F_j\right)\right]} \nonumber\\
		&= L \norm[F]{X\duu{j} + \rho_j\left(F_j'\right) - \rho_j\left(F_j\right)} \nonumber\\
		&\le L \norm[F]{X\duu{j}} + L\norm[F]{\rho_j\left(F_j'\right) - \rho_j\left(F_j\right)} \nonumber\\
		&\le L \norm[F]{X}\norm[2]{\duu{j}} + L^{2}\norm[F]{F_j' - F_j} \nonumber\\
		&\le L\eta M_1 \norm[F]{X} + L^{3}\left(1+\eta\right)M_2 \norm[2]{\pg}\norm[F]{H_{j-1}'-H_{j-1}} + L^{3}\eta M_2 \norm[2]{\pg}\norm[F]{H_{j-1}} \nonumber\\
		&= \tau(1+\eta)\norm[F]{H_{j-1}'-H_{j-1}} + \kappa\eta + \tau\eta\norm[F]{H_{j-1}}. \nonumber
	\end{align}
	We consider different values of $\tau$. If $\tau=1$, we have
	\begin{align}
		\norm[F]{H_{j}'-H_{j}} &\le (1+\eta)\norm[F]{H_{j-1}'-H_{j-1}} + \kappa\eta + \kappa\eta(j-1) \qquad(\mbox{By \cref{eq:pf-lem-pert-f-mpgnn-1}}) \nonumber\\
		&= (1+\eta)\norm[F]{H_{j-1}'-H_{j-1}} + \kappa\eta j \nonumber\\
		&\le \sum_{i=0}^{j-1}\kappa\eta(i+1)\left(1+\eta\right)^{j-i-1} \nonumber\\
		&\le \kappa\eta\left(1+\frac{1}{l}\right)^{j}\sum_{i=0}^{j-1}(i+1)\left(1+\frac{1}{l}\right)^{-i-1} = \kappa\eta\left(1+\frac{1}{l}\right)^{j}\sum_{i=1}^{j}i\left(1+\frac{1}{l}\right)^{-i} \nonumber\\
		&= \kappa\eta\left(1+\frac{1}{l}\right)^{j} \left(1+\frac{1}{l}\right)\inv \frac{1-(j+1)\left(1+\frac{1}{l}\right)^{-j}+j\left(1+\frac{1}{l}\right)^{-j-1}}{\left(1-\left(1+\frac{1}{l}\right)\inv\right)^{2}} \nonumber\\
		&= \kappa\eta \frac{\left(1+\frac{1}{l}\right)^{j+1}-(j+1)\left(1+\frac{1}{l}\right)+j}{\left(\left(1+\frac{1}{l}\right)-1\right)^{2}} \nonumber\\
		&\le \kappa\eta l^2 \left(1+\frac{1}{l}\right)^{j+1} = \kappa\eta l(1+l)\left(1+\frac{1}{l}\right)^{j}. \label{eq:pf-lem-pert-f-mpgnn-2}
	\end{align} 
	If $\tau\ne1$, we have
	\begin{align}
		\norm[F]{H_{j}'-H_{j}} &\le \tau(1+\eta)\norm[F]{H_{j-1}'-H_{j-1}} + \kappa\eta + \tau\eta\norm[F]{H_{j-1}} \nonumber\\
		&\le \tau(1+\eta)\norm[F]{H_{j-1}'-H_{j-1}} + \kappa\eta + \kappa\eta\tau\frac{\tau^{j-1}-1}{\tau-1} \qquad(\mbox{By \cref{eq:pf-lem-pert-f-mpgnn-1}}) \nonumber\\
		&= \tau(1+\eta)\norm[F]{H_{j-1}'-H_{j-1}} + \kappa\eta\frac{\tau^{j}-1}{\tau-1} \nonumber\\
		&\le \sum_{i=0}^{j-1}\kappa\eta\left(\frac{\tau^{i+1}-1}{\tau-1}\right)\tau^{j-i-1}(1+\eta)^{j-i-1} \nonumber\\
		&\le \frac{\kappa\eta\tau^{j}}{\tau-1}\left(1+\frac{1}{l}\right)^{j}\sum_{i=0}^{j-1}\left(1-\tau^{-i-1}\right)\left(1+\frac{1}{l}\right)^{-i-1} \nonumber\\
		&\le \frac{\kappa\eta\tau^{j}}{\tau-1}\left(1+\frac{1}{l}\right)^{j}\sum_{i=1}^{j}\left(1-\tau^{-i}\right). \label{eq:pf-lem-pert-f-mpgnn-3}
	\end{align}
	
	\textbf{Final Bound on the Readout Layer. }
	\begin{align}
		\norm[2]{H_{l}'-H_{l}} &= \norm[2]{\frac{1}{n}\bm{1}_nH_{l-1}'(W_{l}+\dww{l}) - \frac{1}{n}\bm{1}_nH_{l-1}W_{l}} \nonumber\\
		&= \norm[2]{\frac{1}{n}\bm{1}_n \left(H_{l-1}'-H_{l-1}\right)(W_{l}+\dww{l}) + \frac{1}{n}\bm{1}_n H_{l-1}\dww{l}} \nonumber\\
		&\le \norm[2]{\frac{1}{n}\bm{1}_n \left(H_{l-1}'-H_{l-1}\right)(W_{l}+\dww{l})} + \norm[2]{\frac{1}{n}\bm{1}_n H_{l-1}\dww{l}} \nonumber\\
		&\le \norm[2]{\frac{1}{n}\bm{1}_{n}}\norm[2]{\left(H_{l-1}'-H_{l-1}\right)\left(W_{l}+\dww{l}\right)} + \norm[2]{\frac{1}{n}\bm{1}_{n}}\norm[2]{H_{l-1}\dww{l}} \nonumber\\
		&\le \frac{1}{\sqrt{n}}\norm[F]{H_{l-1}'-H_{l-1}}\norm[2]{W_{l}+\dww{l}} + \frac{1}{\sqrt{n}}\norm[F]{H_{l-1}}\norm[2]{\dww{l}} \nonumber\\
		&\le \frac{1}{\sqrt{n}}(1+\eta)M_2 \norm[F]{H_{l-1}'-H_{l-1}} + \frac{1}{\sqrt{n}}\eta M_2 \norm[F]{H_{l-1}}. \nonumber
	\end{align}
	
	If $\tau=1$, combining \cref{eq:pf-lem-pert-f-mpgnn-1,eq:pf-lem-pert-f-mpgnn-2}, we have
	\begin{align}
		\norm[2]{H_{l}'-H_{l}} &\le \frac{1}{\sqrt{n}}(1+\eta)M_2 \norm[F]{H_{l-1}'-H_{l-1}} + \frac{1}{\sqrt{n}}\eta M_2 \norm[F]{H_{l-1}} \nonumber\\
		&\le \frac{1}{\sqrt{n}}(1+\eta)\eta l M_2 \kappa(1+l)\left(1+\frac{1}{l}\right)^{l-1} + \frac{1}{\sqrt{n}}\eta (l-1) \kappa M_2 \nonumber\\
		&\le \frac{1}{\sqrt{n}} \eta l (1+l) \kappa M_2 \left(1+\frac{1}{l}\right)^{l} + \frac{1}{\sqrt{n}}\eta (l-1) \kappa M_2  \nonumber\\
		&= \frac{1}{\sqrt{n}} \kappa \eta M_2 \left(1+\frac{1}{l}\right)^{l} \left(l(1+l)+(l-1)\left(1+\frac{1}{l}\right)^{-l}\right) \nonumber\\
		&\le \frac{1}{\sqrt{n}}e\kappa \eta M_2 \left(l+l^2+l-1\right) \nonumber\\
		&\le \frac{1}{\sqrt{n}}e\kappa \eta M_2(l+1)^2. \nonumber
	\end{align}
	If $\tau\ne1$, combining \cref{eq:pf-lem-pert-f-mpgnn-1,eq:pf-lem-pert-f-mpgnn-3}, we have
	\begin{align}
		\norm[2]{H_{l}'-H_{l}} &\le \frac{1}{\sqrt{n}}(1+\eta)M_2 \norm[F]{H_{l-1}'-H_{l-1}} + \frac{1}{\sqrt{n}}\eta M_2 \norm[F]{H_{l-1}} \nonumber\\
		&\le \frac{1}{\sqrt{n}}(1+\eta)M_2 \frac{\kappa\eta\tau^{l-1}}{\tau-1}\left(1+\frac{1}{l}\right)^{l-1}\sum_{i=1}^{l-1}\left(1-\tau^{-i}\right) + \frac{1}{\sqrt{n}}\eta M_2 \kappa\frac{\tau^{l-1}-1}{\tau-1} \nonumber\\
		&\le \frac{1}{\sqrt{n}}\kappa\eta\tau^{l-1} M_2\left(1+\frac{1}{l}\right)^{l}\left(\sum_{i=1}^{l-1}\frac{1-\tau^{-i}}{\tau-1} + \frac{1-\tau^{1-l}}{\tau-1}\left(1+\frac{1}{l}\right)^{-l}\right) \nonumber\\
		&\le \frac{1}{\sqrt{n}}\kappa\eta\tau^{l-1} M_2\left(1+\frac{1}{l}\right)^{l}\left(\sum_{i=1}^{l-1}\frac{1-\tau^{-i}}{\tau-1} + \frac{1-\tau^{-(l-1)}}{\tau-1}\right) \nonumber\\
		&\le \frac{1}{\sqrt{n}}\kappa\eta\tau^{l-1} M_2\left(1+\frac{1}{l}\right)^{l} l \frac{1-\tau^{-(l-1)}}{\tau-1} \nonumber\\
		&\le \frac{1}{\sqrt{n}}e\kappa \eta l M_2 \frac{\tau^{l-1}-1}{\tau-1}. \nonumber
	\end{align}
	Therefore, 
	\begin{align}
		\norm[2]{\fw{w+u}-\fw{w}} = \norm[2]{H_{l}'-H_{l}} &\le 
		\begin{cases}
			e \eta L M_1  M_2 \norm[2]{X} (l+1)^2, & \tau=1, \\
			e \eta l L M_1 M_2 \norm[2]{X} \dfrac{\tau^{l-1}-1}{\tau-1}, & \tau\ne1. \\
		\end{cases}  \nonumber
	\end{align}
\end{proof}

\begin{lemma}
	\label{lem:pert-mpgnn}
	Under Assumptions~\ref{assumption:graph} to~\ref{assumption:mpgnn}, for any $w$ such that $\norm[2]{U_i}\le M_1, \forall  i\in[l-1], \norm[2]{W_i}\le M_2, \forall i\in[l]$ and any perturbation $\dw=vec\left(\left\{\dww{i}\right\}_{i=1}^{l}\right)$ such that $\eta=\max\left(\left\{\frac{\norm[2]{\duu{i}}}{\norm[2]{U_i}}\right\}_{i=1}^{l-1},\left\{\frac{\norm[2]{\dww{i}}}{\norm[2]{W_i}}\right\}_{i=1}^{l}\right)\le {1\over l}, \forall  i\in[l]$, the change in the margin operator is bounded as 
	\begin{align}
		\left|\mar{\fw{w+\dw}}{i,j}-\mar{\fw{w}}{i,j}\right| \le \begin{cases}
			2e\eta LBM_1M_2 (l+1)^2, & \tau=1,\\
			2e\eta lLBM_1M_2 \dfrac{\tau^{l-1}-1}{\tau-1}, & \tau\ne1,
		\end{cases} \nonumber
	\end{align}
	where $\tau=L^3M_2\norm[2]{\pg}$. Further, consider the $\epsilon$-attack ($\epsilon>0$), the change in the robust margin operator is bounded as 
	\begin{align}
		\left|\mar[RM]{\fw{w+\dw}}{i,j}-\mar[RM]{\fw{w}}{i,j}\right| \le \begin{cases}
			2e\eta L(B+\epsilon)M_1M_2 (l+1)^2, & \tau=1,\\
			2e\eta lL(B+\epsilon)M_1M_2 \dfrac{\tau^{l-1}-1}{\tau-1}, & \tau\ne1.
		\end{cases} \nonumber
	\end{align}
\end{lemma}
\begin{proof}
	Let 
	\begin{align}
		\C = \begin{cases}
			e\eta LM_1M_2 (l+1)^2, & \tau=1,\\
			e\eta lLM_1M_2 \dfrac{\tau^{l-1}-1}{\tau-1}, & \tau\ne1.
		\end{cases} \nonumber
	\end{align}
	For any $i,j\in[K]$, by the definition of margin operator, we have
	\begin{align}
		\left|\mar{\fw{w+\dw}}{i,j}-\mar{\fw{w}}{i,j}\right| &\le \left|\fw{w+\dw}_i-\fw{w}_i\right|  + \left|\fw{w+\dw}_j-\fw{w}_j\right|\nonumber\\
		&\le 2\norm[2]{\fw{w+\dw}-\fw{w}}. \nonumber
	\end{align}
	By \cref{prop:graph-norm,lem:pert-f-mpgnn}, we have
	\begin{align}
		\left|\mar{\fw{w+\dw}}{i,j}-\mar{\fw{w}}{i,j}\right| &\le 2\norm[2]{X}\C \le 2B\C. \nonumber
	\end{align}
	Further, let
	\begin{align}
		G(w) &= \mathop{\arg\inf}_{G'\in\delta_w(G)}\left(\fw[G']{w}_i - \fw[G']{w}_j\right) = \mathop{\arg\inf}_{G'\in\delta_w(G)} \mar{\fw{w}}{i,j}, \nonumber
	\end{align}
	and $X(w)$ is the corresponding node features matrix. Denoting $\tilde{w}=w+\dw$, we have
	\begin{align}
		&\quad \, \left|\mar[RM]{\fw{\tilde{w}}}{i,j}-\mar[RM]{\fw{w}}{i,j}\right| \nonumber\\
		&= \left|\inf_{G'\in\delta_{\tilde{w}}(G)}\mar{\fw{\tilde{w}}}{i,j} - \inf_{G'\in\delta_{w}(G)}\mar{\fw{w}}{i,j}\right| \nonumber\\
		&= \left|\mar{\fw[G(\tilde{w})]{\tilde{w}}}{i,j} - \mar{\fw[G(w)]{w}}{i,j}\right| \nonumber\\
		&\le \max\left\{\left|\mar{\fw[G(\tilde{w})]{\tilde{w}}}{i,j} - \mar{\fw[G(\tilde{w})]{w}}{i,j}\right|,\right. \nonumber\\
		&\qquad\qquad \left.\left|\mar{\fw[G(w)]{\tilde{w}}}{i,j} - \mar{\fw[G(w)]{w}}{i,j}\right|\right\} \nonumber\\
		&\le 2\max\left\{\norm[2]{X(\tilde{w})}, \norm[2]{X(w)}\right\}\C \nonumber\\
		&\le 2(B+\epsilon)\C, \nonumber
	\end{align}
	where the first inequality holds since, by the definition of $G(w)$,
	\begin{align}
		& \mar{\fw[G(\tilde{w})]{\tilde{w}}}{i,j} - \mar{\fw[G(w)]{w}}{i,j}\le \mar{\fw[G(w)]{\tilde{w}}}{i,j} - \mar{\fw[G(w)]{w}}{i,j}, \nonumber\\
		& \mar{\fw[G(w)]{w}}{i,j} - \mar{\fw[G(\tilde{w})]{\tilde{w}}}{i,j} \le \mar{\fw[G(\tilde{w})]{w}}{i,j} - \mar{\fw[G(\tilde{w})]{\tilde{w}}}{i,j}, \nonumber
	\end{align}
	and the last inequality holds from Assumption~\ref{assumption:graph} and the definition of $\epsilon$-attack.
\end{proof}

\subsubsection{Proof of Generalization bounds for MPGNN} 
\label{pf:thm-adv-gen-mpgnn}

\begin{proof}[Proof of \cref{thm:adv-gen-mpgnn}]
	We consider the cases of different values of $\tau$.
	
	\textbf{Case $\tau\ne1$}\quad Let $\beta=\max\left\{\frac{1}{\zeta},\xi^{1/l}\right\}$, where $\zeta=\min\left(\left\{\norm[2]{U_{i}}\right\}_{i=1}^{l-1},\left\{\norm[2]{W_{i}}\right\}_{i=1}^{l}\right)$ and $\xi=LM_1M_2\dfrac{\tau^{l-1}-1}{\tau-1}$.
	Consider the prior distribution $P = \N(0,\sigma^2I)$ and the random perturbation distribution $u\sim \N(0,\sigma^2I)$. Note that the $\sigma$ of the prior and the perturbation distributions are the same and will be set according to $\beta$. More precisely, we will set the $\sigma$ based on some approximation $\tilde{\beta}$ of $\beta$ since the prior distribution $P$ cannot depend on any learned weights directly. The approximation $\tilde{\beta}$ is chosen to be a cover set which covers the meaningful range of $\beta$. For now, let us assume that we have a fix $\tilde{\beta}$ and consider $\beta$ which satisfies $|\beta-\tilde{\beta}|\le\frac{1}{l+2}\beta$. Note that this also implies
	\begin{align}
		|\beta-\tilde{\beta}|\le\frac{1}{l+2}\beta &\ \Rightarrow\ \left(1-\frac{1}{l+2}\right)\beta\le \tilde{\beta}\le \left(1+\frac{1}{l+2}\right)\beta \nonumber\\
		&\ \Rightarrow\ \left(1-\frac{1}{l+2}\right)^{l+1} \beta^{l+1} \le \tilde{\beta}^{l+1} \le \left(1+\frac{1}{l+2}\right)^{l+1} \beta^{l+1} \nonumber\\
		&\ \Rightarrow\ \left(1-\frac{1}{l+2}\right)^{l+1} \beta^{l+1} \le \tilde{\beta}^{l+1} \le \left(1+\frac{1}{l+2}\right)^{l+2} \beta^{l+1} \nonumber\\
		&\ \Rightarrow\ \frac{1}{e}\beta^{l+1} \le \tilde{\beta}^{l+1} \le e\beta^{l+1}. \nonumber
	\end{align}
	Denoting $\dww{i}\in\R^{h_{i-1}\times h_{i}}, \duu{i}\in\R^{h_{0}\times h_{i}}$ as $\Delta W$, from \citep{tropp2012user}, for $\Delta W \sim\N(0,\sigma^2I)$, we have
	\begin{align}
		\Prob\left\{\norm[2]{\Delta W} \ge t\right\} \le 2h\exp\left(-\frac{t^2}{2h\sigma^2}\right). \nonumber
	\end{align}
	Taking a union bound, we have
	\begin{align}
		&\quad\, \Prob\left(\bigcap_{i=1}^{l-1}\left\{\norm[2]{\duu{i}}<t\right\}\cap\bigcap_{i=1}^{l}\left\{\norm[2]{\dww{i}}<t\right\}\right) \nonumber\\
		&= 1-\Prob\left(\bigcup_{i=1}^{l-1}\left\{\norm[2]{\duu{i}}\ge t\right\}\cup\bigcup_{i=1}^{l}\left\{\norm[2]{\dww{i}}\ge t\right\}\right) \nonumber\\
		&\ge 1 - \sum_{i=1}^{l-1}\Prob\left\{\norm[2]{\duu{i}} \ge t\right\} - \sum_{i=1}^{l}\Prob\left\{\norm[2]{\dww{i}} \ge t\right\} \nonumber\\
		&\ge 1 - 2(2l-1)h\exp\left(-\frac{t^2}{2h\sigma^2}\right). \nonumber
	\end{align}
	Setting $2(2l-1)h\exp\left(-\frac{t^2}{2h\sigma^2}\right)=\frac{1}{2}$, we have $t=\sigma\sqrt{2h\log(4(2l-1)h)}$. This implies that 
	\begin{align}
		\Prob\left(\max\left(\left\{\norm[2]{\duu{i}}\right\}_{i=1}^{l-1},\left\{\norm[2]{\dww{i}}\right\}_{i=1}^{l}\right) < t \right) \ge \frac{1}{2}, \nonumber
	\end{align}
	and thus,
	\begin{align}
		\Prob\left(\max\left(\left\{\frac{\norm[2]{\duu{i}}}{\norm[2]{U_{i}}}\right\}_{i=1}^{l-1},\left\{\frac{\norm[2]{\dww{i}}}{\norm[2]{W_{i}}}\right\}_{i=1}^{l}\right) < \dfrac{t}{\zeta} \right) \ge \frac{1}{2}. \nonumber
	\end{align}
	Plugging this bound into \cref{lem:pert-mpgnn}, we have that with probability at least $\frac{1}{2}$ over $\dw$,
	\begin{align}
		\left|\mar[R]{\fw{w\dw}}{i,j} - \mar[R]{\fw{w}}{i,j}\right| &\le 2el(B+\epsilon) L M_1 M_2 \dfrac{\tau^{l-1}-1}{\tau-1} \frac{t}{\zeta} \nonumber\\
		&\le 2el(B+\epsilon) \beta^{l+1} t \nonumber\\
		&\le 2e^{2}l(B+\epsilon) \tilde{\beta}^{l+1} \sigma\sqrt{2h\log(4(2l-1)h)} \le \frac{\gamma}{2}, \nonumber
	\end{align}
	where we set $\sigma=\frac{\gamma}{4e^{2}l(B+\epsilon) \tilde{\beta}^{l+1} \sqrt{2h\log(4(2l-1)h)}}$ to get the last inequality. Note that \cref{lem:pert-mpgnn} also requires $\forall i\in[l], \norm[2]{\dww{i}}\le \frac{1}{l}\norm[2]{W}, \norm[2]{\duu{i}}\le \frac{1}{l}\norm[2]{U}$. The requirement is satisfied if $\sigma\le \frac{1}{\zeta l\sqrt{2h\log(4(2l-1)h)}}$ which in turn can be satisfied if 
	\begin{align}\label{eq:pf-mpgnn-bound-1}
		\frac{\gamma}{4e^{2}l(B+\epsilon) \tilde{\beta}^{l+1} \sqrt{2h\log(4(2l-1)h)}} \le \frac{1}{\beta l\sqrt{2h\log(4(2l-1)h)}},
	\end{align}
	and one of the sufficient conditions for this is $\frac{\gamma}{4e(B+\epsilon)} \le \beta^{l}$. We will see how to satisfy this condition later.
	
	\noindent We now compute the KL term in the PAC-Bayes bound in \cref{lem:pert-kl-bound}. As defined above, for a fixed $w$, the prior distribution $P=\N(0,\sigma^2I)$, the posterior distribution $Q$ is the distribution of $\dw$ shifted by $w$, that is, $Q=\N(w,\sigma^2I)$. Then we have,
	\begin{align}
		\kl{Q} = \frac{\norm[2]{w}^2}{2\sigma^2} &= \frac{16e^{4}l^{2}(B+\epsilon)^{2} \tilde{\beta}^{2l+2} 2h\log(4(2l-1)h)}{2\gamma^2} \left(\sum_{i=1}^{l-1}\norm[F]{U_{i}}^2+\sum_{i=1}^{l}\norm[F]{W_{i}}^2\right) \nonumber\\
		&\le \cO\left(\frac{(B+\epsilon)^{2} \beta^{2l+2} l^{2}h\log(lh)}{\gamma^2} \left(\sum_{i=1}^{l-1}\norm[F]{U_{i}}^2+\sum_{i=1}^{l}\norm[F]{W_{i}}^2\right)\right). \nonumber
	\end{align}
	From \cref{lem:pert-kl-bound}, fixing any $\tilde{\beta}$, with probability $1-\delta$ and for all $w$ such that $|\beta-\tilde{\beta}|\le\frac{1}{l+2}\beta$, we have,
	\begin{align}\label{eq:pf-mpgnn-bound-2}
		\ml{R}{0} \le \eml{R}{\gamma} + \cO\left(\sqrt{\frac{(B+\epsilon)^2 \beta^{2l+2} l^2 h \log(lh)\norm[2]{w}^2 + \log(m/\delta)}{\gamma^2 m}}\right).
	\end{align}
	\hspace*{\fill}\\
	\noindent Finally, we need to consider multiple choices of $\tilde{\beta}$ so that for any $\beta$, we can bound the generalization error like \cref{eq:pf-mpgnn-bound-2}. First, we only need to consider values of $\beta$ in the following range,
	\begin{align}\label{eq:pf-mpgnn-bound-3}
		\left(\frac{\gamma}{2(B+\epsilon)}\right)^{1/l} \le \beta \le \left(\frac{\gamma\sqrt{m}}{2(B+\epsilon)}\right)^{1/l},
	\end{align}
	since otherwise the bound holds trivially as $\ml{R}{0}\le 1$ by definition. To see this, if $\beta^{l}<\frac{\gamma}{2(B+\epsilon)}$, then for any $G$ and any $i, j\in[K]$, we have,
	\begin{align}
		\left|\mar[RM]{\fw{w}}{i,j}\right| &= \left|\inf_{G'\in\delta_{w}(G)}\mar{\fw[G']{w}}{i,j}\right| = \left|\mar{\fw[\hat{G}]{w}}{i,j}\right| \nonumber\\
		&= \left|\fw[\hat{G}]{w}_{i} -\fw[\hat{G}]{w}_{j}\right| \le 2\left|\fw[\hat{G}]{w}_{i}\right| \nonumber\\
		&\le 2\norm[2]{\fw[\hat{G}]{w}} = 2\norm[2]{\frac{1}{n}\bm{1}_n H_{l-1}^*W_l} \le 2\norm[2]{\frac{1}{n}\bm{1}_n} \norm[F]{H_{l-1}^*W_l} \nonumber\\
		&\le \frac{2}{\sqrt{n}} \norm[F]{H_{l-1}^*}\norm[2]{W_l} \le 2(B+\epsilon) LM_1 M_2 \dfrac{\tau^{l-1}-1}{\tau-1} \nonumber\\
		&\le 2(B+\epsilon) \beta^{l} \le 2(B+\epsilon) \frac{\gamma}{2(B+\epsilon)} = \gamma, \nonumber 
	\end{align}
	where $\hat{G}=\arg\inf_{G'\in\delta_{w}(G)}\mar{\fw[G']{w}}{i,j}$ and $H_{l-1}^*$ is the corresponding node representation.
	Therefore, by the definition, we always have $\eml{L}{\gamma}=1$ when $\beta^{l}<\frac{\gamma}{2(B+\epsilon)}$.
	
	\noindent Alternatively, if $\beta^{l}>\frac{\gamma\sqrt{m}}{2(B+\epsilon)}$, the term inside the big-O notation in \cref{eq:pf-mpgnn-bound-2} would be
	\begin{align}
		\sqrt{\frac{(B+\epsilon)^2 \beta^{2l+2} l^2 h \log(lh)\norm[2]{w}^2 + \log(m/\delta)}{\gamma^2 m}} &\ge \sqrt{\frac{l^2 h \log(lh) \left(\norm[2]{w}^2/\zeta^{2}\right) + \log(m/\delta)}{4}} \ge 1, \nonumber
	\end{align}	
	where the first inequality holds since $\beta^{2}\ge\zeta^{2}$ and the last inequality holds since 
	\begin{align}
		\norm[2]{w}^2\ge \min\left(\left\{\norm[F]{U_{i}}^{2}\right\}_{i=1}^{l-1}, \left\{\norm[F]{W_{i}}^{2}\right\}_{i=1}^{l}\right) \ge \zeta^{2}, \nonumber
	\end{align}
	and we typically choose $h\ge2$ in practice and $l\ge2$.
	\hspace*{\fill}\\
	\noindent And note that the lower bound in \cref{eq:pf-mpgnn-bound-3} ensures that \cref{eq:pf-mpgnn-bound-1} holds which in turn justifies the applicability of \cref{lem:pert-mpgnn}. 
	
	Consider a covering set centered at $C=\left\{\tilde{\beta}_1,\dots,\tilde{\beta}_{|C|}\right\}$ of the interval in \cref{eq:pf-mpgnn-bound-3} with radius $\frac{1}{l+2}\left(\frac{\gamma}{2(B+\epsilon)}\right)^{1/l}$, where $|C|$ is the size of the set. To make sure that $|\beta-\tilde{\beta}|\le\frac{1}{l+2}\left(\frac{\gamma}{2(B+\epsilon)}\right)^{1/l}\le\frac{1}{l+2}\beta$, the minimal size of $C$ should be
	\begin{align}
		|C| = \frac{\left(\frac{\gamma\sqrt{m}}{2(B+\epsilon) }\right)^{1/l} - \left(\frac{\gamma}{2(B+\epsilon)}\right)^{1/l}}{\frac{2}{l+2}\left(\frac{\gamma}{2(B+\epsilon)}\right)^{1/l}} = \frac{l+2}{2}\left(m^{\frac{1}{2l}}-1\right). \nonumber
	\end{align}
	
	For a $\tilde{\beta}_i\in C$, we denote the event that \cref{eq:pf-mpgnn-bound-2} holds for $|\beta-\tilde{\beta}|\le\frac{1}{l}\beta$ as $E_i$. Therefore, the probability of the event that \cref{eq:pf-mpgnn-bound-2} holds for any $\beta$ satisfying \cref{eq:pf-mpgnn-bound-3} is
	\begin{align}
		\Prob\left(\bigcap_{i=1}^{|C|}E_i\right) = 1-\Prob\left(\bigcup_{i=1}^{|C|}\bar{E}_i\right) \ge 1-\sum_{i=1}^{|C|}\Prob\left\{\bar{E}_i\right\} \ge 1-|C|\delta, \nonumber
	\end{align}
	where $\bar{E}_i$ denotes the complement of $E_i$. Hence, let $\delta'=|C|\delta$, then we have that, with probability $1-\delta'$ and for all $w$,
	\begin{align}
		\ml{R}{0} &\le \eml{R}{\gamma} + \cO\left(\sqrt{\frac{(B+\epsilon)^2 \beta^{2l+2} l^2 h \log(lh)\norm[2]{w}^2 + \log(m/\delta)}{\gamma^2 m}}\right) \nonumber\\
		&\le \eml{R}{\gamma} \nonumber\\
		&\quad \cO\left(\sqrt{\frac{(B+\epsilon)^2 \max\left\{\zeta\inv, \xi^{1/l}\right\}^{2l+2} l^2 h \log(lh)\norm[2]{w}^2 + \log(m(l+2)/\delta')}{\gamma^2 m}}\right), \nonumber
	\end{align}
	where $\zeta=\min\left(\left\{\norm[2]{U_{i}}\right\}_{i=1}^{l-1},\left\{\norm[2]{W_{i}}\right\}_{i=1}^{l}\right)$, $\xi=LM_1M_2\dfrac{\tau^{l-1}-1}{\tau-1}$.

	\hspace*{\fill}\\
	\noindent \textbf{Case $\tau=1$}\quad We now prove the case of $\tau=1$, which is similar to the case of $\tau\ne1$. Note that this case happens rarely in practice. We only include it for the completeness of the analysis. Again let $\beta=\max\left\{\frac{1}{\zeta}, \sqrt{LM_1M_2}\right\}$. For now, let us assume that we have a fix $\tilde{\beta}$ and consider $\beta$ which satisfies $|\beta-\tilde{\beta}|\le\frac{1}{4}\beta$. Note that this also implies $\frac{1}{e}\beta^{3}\le \tilde{\beta}^{3}\le e\beta^{3}$. Based on \cref{lem:pert-mpgnn}, we have with probability at least $\frac{1}{2}$ over $\dw$,
	\begin{align}
		\left|\mar[R]{\fw{w+\dw}}{i,j} - \mar[R]{\fw{w}}{i,j}\right| &\le 2 e(l+1)^2(B+\epsilon) L M_1 M_2 \frac{t}{\zeta}  \nonumber\\
		&\le 2 e(l+1)^2(B+\epsilon) \beta^{3} t \nonumber\\
		&\le 2 e^2(l+1)^2(B+\epsilon) \tilde{\beta}^{3} \sigma \sqrt{2h\log(4(2l-1)h)} \le \frac{\gamma}{2}, \nonumber
	\end{align}
	where we set $\sigma=\frac{\gamma}{4e^{2}(l+1)^2(B+\epsilon) \tilde{\beta}^{3} \sqrt{2h\log(4(2l-1)h)}}$ to get the last inequality. Note that \cref{lem:pert-mpgnn} also requires $\forall i\in[l], \norm[2]{U_i}\le \frac{1}{l}\norm[2]{W_i}$. The requirement is satisfied if $\sigma\le \frac{1}{\zeta l\sqrt{2h\log(4(2l-1)h)}}$ which in turn can be satisfied if 
	\begin{align}\label{eq:pf-mpgnn-bound-4}
		\frac{\gamma}{4e^{2}(l+1)^2(B+\epsilon) \tilde{\beta}^{3} \sqrt{2h\log(4(2l-1)h)}} \le \frac{1}{\beta l\sqrt{2h\log(4(2l-1)h)}}, 
	\end{align}
	and one of the sufficient conditions is $\frac{l\gamma}{4e(B+\epsilon)(l+1)^2} \le \beta^{2}$. We will see how to satisfy this condition later.
	
	\noindent We now compute the KL term in the PAC-Bayes bound in \cref{lem:pert-kl-bound}. As defined above, for a fixed $w$, the prior distribution $P=\N(0,\sigma^2I)$, the posterior distribution $Q$ is the distribution of $\dw$ shifted by $w$, that is, $Q=\N(w,\sigma^2I)$. Then we have,
	\begin{align}
		\kl{Q} = \frac{\norm[2]{w}^2}{2\sigma^2} &= \frac{16e^{4}(l+1)^{4}(B+\epsilon)^{2} \tilde{\beta}^{6} 2h\log(4(2l-1)h)}{2\gamma^2} \left(\sum_{i=1}^{l-1}\norm[F]{U_{i}}^2+\sum_{i=1}^{l}\norm[F]{W_{i}}^2\right) \nonumber\\
		&\le \cO\left(\frac{(B+\epsilon)^{2} \beta^{6} (l+1)^{4}h\log(lh)}{\gamma^2} \left(\sum_{i=1}^{l-1}\norm[F]{U_{i}}^2+\sum_{i=1}^{l}\norm[F]{W_{i}}^2\right)\right). \nonumber
	\end{align}
	From \cref{lem:pert-kl-bound}, fixing any $\tilde{\beta}$, with probability $1-\delta$ and for all $w$ such that $|\beta-\tilde{\beta}|\le\frac{1}{4}\beta$, we have,
	\begin{align}\label{eq:pf-mpgnn-bound-5}
		\ml{R}{0} \le \eml{R}{\gamma} + \cO\left(\sqrt{\frac{(B+\epsilon)^2 \beta^{6} (l+1)^4 h \log(lh)\norm[2]{w}^2 + \log(m/\delta)}{\gamma^2 m}}\right).
	\end{align}
	\hspace*{\fill}\\
	\noindent Finally, we need to consider multiple choices of $\tilde{\beta}$ so that for any $\beta$, we can bound the generalization error like \cref{eq:pf-mpgnn-bound-5}. First, we only need to consider values of $\beta$ in the following range,
	\begin{align}\label{eq:pf-mpgnn-bound-6}
		\sqrt{\frac{\gamma}{2(B+\epsilon)l}} \le \beta \le \sqrt{\frac{\gamma\sqrt{m}}{2(B+\epsilon)l}},
	\end{align}
	since otherwise the bound holds trivially as $\ml{R}{0}\le 1$ by definition. To see this, if $\beta^{2}<\frac{\gamma}{2(B+\epsilon)l}$, then for any $(X,A)$ and any $i, j\in[K]$, we have,
	\begin{align}
		\left|\mar[RM]{\fw{w}}{i,j}\right| &= \left|\inf_{G'\in\delta_{w}(G)}\mar{\fw[G']{w}}{i,j}\right| = \left|\mar{\fw[\hat{G}]{w}}{i,j}\right| \nonumber\\
		&= \left|\fw[\hat{G}]{w}_{i} -\fw[\hat{G}]{w}_{j}\right| \le 2\left|\fw[\hat{G}]{w}_{i}\right| \nonumber\\
		&\le 2\norm[2]{\fw[\hat{G}]{w}} = 2\norm[2]{\frac{1}{n}\bm{1}_n H_{l-1}^*W_l} \le 2\norm[2]{\frac{1}{n}\bm{1}_n} \norm[F]{H_{l-1}^*W_l} \nonumber\\
		&\le \frac{2}{\sqrt{n}} \norm[F]{H_{l-1}^*}\norm[2]{W_l} \le 2(B+\epsilon) (l-1) L M_1 M_2 \le 2(B+\epsilon) (l-1) \beta^{2} \nonumber\\
		&\le 2(B+\epsilon)(l-1) \frac{\gamma}{2(B+\epsilon)l} \leq \gamma, \nonumber
	\end{align}
	where $\hat{G}=\arg\inf_{G'\in\delta_{w}(G)}\mar{\fw[G']{w}}{i,j}$ and $H_{l-1}^*$ is the corresponding node representation.
	Therefore, by the definition, we always have $\eml{R}{\gamma}=1$ when $\beta^{2}<\frac{\gamma}{2(B+\epsilon)l}$.
	
	\noindent Alternatively, if $\beta^{2}>\frac{\gamma\sqrt{m}}{2(B+\epsilon)l}$, the term inside the big-O notation in \cref{eq:pf-mpgnn-bound-5} would be
	\begin{align}
		\sqrt{\frac{(B+\epsilon)^2 \beta^{6} (l+1)^4 h \log(lh)\norm[2]{w}^2 + \log(m/\delta)}{\gamma^2 m}} &\ge \sqrt{\frac{(l+1)^4 h \log(lh) \beta^{2}\norm[2]{w}^2 + \log(m/\delta)}{4l^2}} \ge 1, \nonumber
	\end{align}	
	where  the last inequality holds since 
	\begin{align}
		\norm[2]{w}^2\ge \min\left(\left\{\norm[F]{U_{i}}^{2}\right\}_{i=1}^{l-1}, \left\{\norm[F]{W_{i}}^{2}\right\}_{i=1}^{l}\right) \ge \zeta^{2}, \nonumber
	\end{align}
	and we typically choose $h\ge2$ in practice and $l\ge2$.
	\hspace*{\fill}\\
	\noindent And note that the lower bound in \cref{eq:pf-mpgnn-bound-6} ensures that \cref{eq:pf-mpgnn-bound-4} holds which in turn justifies the applicability of \cref{lem:pert-mpgnn}. 
	
	Consider a covering set $C=\left\{\tilde{\beta}_1,\dots,\tilde{\beta}_{|C|}\right\}$ of the interval in \cref{eq:pf-mpgnn-bound-6} with radius $\frac{1}{4}\sqrt{\frac{\gamma}{2(B+\epsilon)l}}$, where $|C|$ is the size of the set. To make sure that $|\beta-\tilde{\beta}|\le\frac{1}{4}\sqrt{\frac{\gamma}{2(B+\epsilon)l}}\le\frac{1}{4}\beta$, the minimal size of $C$ should be
	\begin{align}
		|C| = \frac{\sqrt{\frac{\gamma\sqrt{m}}{2(B+\epsilon)l}}-\sqrt{\frac{\gamma}{2(B+\epsilon)l}}}{\frac{2}{4}\sqrt{\frac{\gamma}{2(B+\epsilon)l}}} = 2\left(m^{\frac{1}{4}}-1\right). \nonumber
	\end{align}
	
	For a $\tilde{\beta}_i\in C$, we denote the event that \cref{eq:pf-mpgnn-bound-5} holds for $|\beta-\tilde{\beta}|\le\frac{1}{l}\beta$ as $E_i$. Therefore, the probability of the event that \cref{eq:pf-mpgnn-bound-5} holds for any $\beta$ satisfying \cref{eq:pf-mpgnn-bound-6} is
	\begin{align}
		\Prob\left(\bigcap_{i=1}^{|C|}E_i\right) = 1-\Prob\left(\bigcup_{i=1}^{|C|}\bar{E}_i\right) \ge 1-\sum_{i=1}^{|C|}\Prob\left\{\bar{E}_i\right\} \ge 1-|C|\delta, \nonumber
	\end{align}
	where $\bar{E}_i$ denotes the complement of $E_i$. Hence, then we have that, with probability $1-\delta$ and for all $w$,
	\begin{align}
		\ml{L}{0} &\le \eml{L}{\gamma} + \cO\left(\sqrt{\frac{(B+\epsilon)^2 \beta^{6} (l+1)^4 h \log(lh)\norm[2]{w}^2 + \log(m|C|/\delta)}{\gamma^2 m}}\right) \nonumber\\
		&\le \eml{L}{\gamma} \nonumber\\
		&\quad + \cO\left(\sqrt{\frac{(B+\epsilon)^2 \max\left\{\zeta^{-6}, \left(LM_1M_2\right)^{3}\right\} (l+1)^4 h \log(lh)\norm[2]{w}^2 + \log(m/\delta)}{\gamma^2 m}}\right). \nonumber
	\end{align}
\end{proof}

\bibliographystyle{elsarticle-harv} 
\bibliography{ref}



%
%
%
\end{document}